\documentclass{article}
\usepackage[accepted]{icml2025}

\usepackage[T1]{fontenc}
\usepackage[utf8]{inputenc}
\usepackage[english]{babel}
\usepackage{mathtools}
\usepackage{url}
\usepackage{nicefrac}
\usepackage{bm}
\usepackage{float}

\usepackage{microtype}
\usepackage{graphicx}
\usepackage{subcaption}
\usepackage{booktabs}
\usepackage{ upgreek }
\usepackage{IEEEtrantools}

\usepackage[inline]{enumitem}
\usepackage{pgfplots}
\usepackage{tikz}

\usepackage{amsbsy, amsmath, amsthm, amsfonts, amssymb}
\usepackage{cases}
\usepackage{xcolor}

\usepackage{natbib}
\usepackage{hyperref}
\usepackage{cleveref}
\hypersetup{
  colorlinks=true,
  linkcolor=blue,
  filecolor=magenta,
  urlcolor=cyan,
  citecolor=blue
}
\usepackage{tabularx}

\usepackage{svg}

\usepackage{scalerel,stackengine}
\stackMath
\newcommand\reallywidehat[1]{%
\savestack{\tmpbox}{\stretchto{%
  \scaleto{%
    \scalerel*[\widthof{\ensuremath{#1}}]{\kern-.6pt\bigwedge\kern-.6pt}%
    {\rule[-\textheight/2]{1ex}{\textheight}}
  }{\textheight}%
}{0.5ex}}%
\stackon[1pt]{#1}{\tmpbox}%
}
\newenvironment{talign*}
 {\csname align*\endcsname}
 {\endalign}

\newtheorem{theorem}{Theorem}

\newtheorem{connection}{Connection}
\crefname{connection}{Connection}{Connections}
\newtheorem{conjecture}{Conjecture}
\crefname{conjecture}{Conjecture}{Conjectures}
\newtheorem{lemma}{Lemma}
\newtheorem{proposition}{Proposition}
\newtheorem{assumption}{Assumption}

\crefname{assumption}{}{}

\theoremstyle{definition}
\newtheorem{definition}{Definition}

\def\Re{\mathbb{R}}

\def\bE{\mathbb{E}}

\def\bU{\mathbb{U}}

\def\calX{\mathcal{X}}

\def\calH{\mathcal{H}}

\def\calP{\mathcal{P}}

\def\calF{\mathcal{F}}
\def\calN{\mathcal{N}}
\def\calB{\mathcal{B}}

\def\Id{\mathrm{Id}}

\def\MMD{\mathrm{MMD}}
\def\ekqd{\text{e-KQD}}
\def\supkqd{\text{sup-KQD}}

\def\argmin{\mathop{\text{argmin}}}
\def\argmax{\mathop{\text{argmax}}}

\def\d{\text{d}}

\newcommand{\bigo}{\mathcal{O}}

\newcommand{\masha}[1]{\textcolor{teal}{}}

\crefname{enumi}{}{}
\crefname{enumii}{}{}
\crefname{corollary}{Corollary}{Corollaries}

\usepackage{xr}
\makeatletter

\newcommand*{\addFileDependency}[1]{
\typeout{(#1)}
\@addtofilelist{#1}
\IfFileExists{#1}{}{\typeout{No file #1.}}
}\makeatother

\newcommand{\eqcolon}{\mathrel{\resizebox{\widthof{$\mathord{=}$}}{\height}{ $\!\!=\!\!\resizebox{1.2\width}{0.8\height}{\raisebox{0.23ex}{$\mathop{:}$}}\!\!$ }}}

\begin{document}
\twocolumn[

\icmltitle{Kernel Quantile Embeddings and Associated Probability Metrics}

\icmlsetsymbol{equal}{*}

\begin{icmlauthorlist}
\icmlauthor{Masha Naslidnyk}{uclcs}
\icmlauthor{Siu Lun Chau}{ntu}
\icmlauthor{Fran\c{c}ois-Xavier Briol}{uclstats}
\icmlauthor{Krikamol Muandet}{cispa}
\end{icmlauthorlist}

\icmlaffiliation{ntu}{College of Computing \& Data Science, Nanyang Technological University, Singapore}
\icmlaffiliation{uclcs}{Department of Computer Science, University College London, London, UK}
\icmlaffiliation{cispa}{CISPA Helmholtz Center for Information Security, Saarbr\"{u}cken, Germany}
\icmlaffiliation{uclstats}{Department of Statistical Science, University College London, London, UK}

\icmlcorrespondingauthor{Masha Naslidnyk}{masha.naslidnyk.21@ucl.ac.uk}

\icmlkeywords{Machine Learning, ICML}

\vskip 0.3in
]
\pagestyle{fancy}

\printAffiliationsAndNotice{}

\begin{abstract}
Embedding probability distributions into reproducing kernel Hilbert spaces (RKHS) has enabled powerful nonparametric methods such as the maximum mean discrepancy (MMD), a statistical distance with strong theoretical and computational properties. At its core, the MMD relies on kernel mean embeddings to represent distributions as mean functions in RKHS. However, it remains unclear if the mean function is the only meaningful RKHS representation.
Inspired by generalised quantiles, we introduce the notion of \emph{kernel quantile embeddings (KQEs)}. We then use KQEs to construct a family of distances that:
(i) are probability metrics under weaker kernel conditions than MMD;
(ii) recover a kernelised form of the sliced Wasserstein distance; and
(iii) can be efficiently estimated with near-linear cost.
Through hypothesis testing, we show that these distances offer a competitive alternative to MMD and its fast approximations.
\end{abstract}

\section{Introduction}
Many machine learning and statistical methods rely on representing, comparing, and measuring the distance between probability distributions.
Kernel mean embeddings (KMEs) have been shown to be a mathematically and computationally convenient approach for this task \citep{Berlinet2004,Smola07Hilbert,Muandet2016}. At its core, a KME represents a distribution as a mean function in a reproducing kernel Hilbert space (RKHS). When the kernel function is sufficiently regular and satisfies a condition called `characteristic' \citep{Sriperumbudur2009}, the representation of a distribution as a KME is unique, capturing all information about the distribution.
The probability metric constructed by comparing KMEs, called maximum mean discrepancy (MMD) \citep{Borgwardt06:MMD,gretton2012kernel}, has received significant attention due to its computational tractability. Its most common estimator has cost $\bigo(n^2)$ and can be estimated with error $\bigo(n^{-\nicefrac{1}{2}})$ in the number of data points $n$, but cheaper alternatives have also been proposed \cite{gretton2012kernel,Chwialkowski2015,Bodenham2023,Schrab2022}.
For this reason, KMEs and the MMD have been used to tackle a broad range of tasks from hypothesis testing~\citep{gretton2012kernel} to parameter estimation~\citep{briol2019statistical,Cherief-Abdellatif2020_MMDBayes}, causal inference~\citep{muandet2021counterfactual,Sejdinovic2024}, feature attribution~\citep{chau2022rkhs,chau2023explaining}, and learning on distributions~\citep{Muandet12:SMM,szabo2016learning}.

Nevertheless, the question of whether alternative kernel-based embeddings, particularly nonlinear counterparts, could exhibit desirable properties has long remained under-explored, in part due to the associated computational challenges. Recently, this gap has begun to be addressed, with works investigating kernelised medians \citep{Nienkotter2023}, cumulants \citep{Bonnier2023}, and variances \citep{Makigusa2024}. In this paper, we consider an alternative based on the concept of quantiles in an RKHS, which we term \emph{kernel quantile embeddings (KQEs)}. Similarly to the construction of KMEs, KQEs are obtained by considering the directional quantiles of a feature map obtained from a reproducing kernel. KQEs also lead naturally to a family of distances which we call \emph{kernel quantile discrepancies (KQDs)}. This approach is motivated from the statistics and econometrics literature~\citep{Kosorok1999,dominicy2013method,ranger2020minimum,stolfi2022sparse}, where matching quantiles has been shown effective in constructing statistical estimators and hypothesis tests.

Our paper identifies several desirable properties of KQEs.
Firstly, from a theoretical point of view, we show in \Cref{res:cramer-wold} and \Cref{res:if_meanchar_then_quantchar} that KQEs can represent distributions on any space for which we can define a kernel, and that the conditions to make a kernel \emph{quantile-characteristic}, that is for KQEs to be a one-to-one representation of a probability distribution, are weaker than for the classical notion of characteristic, which we now call \emph{mean-characteristic}. We then show in \Cref{res:consistency_KQE} that KQEs can be estimated at a rate of $\bigo(n^{-\nicefrac{1}{2}})$ in the number of samples $n$; the same rate as that of the empirical estimator of KMEs~\citep{tolstikhin2017minimax}. As a result, KQDs are probability metrics under much weaker conditions than the MMD (see \Cref{res:KQE_characterise_dists}), while maintaining comparable computational guarantees, including a finite-sample consistency with rate $\bigo(n^{-\nicefrac{1}{2}})$ (up to log terms) for their empirical estimators (see \Cref{res:consistency_KQD}).

Secondly, we establish a number of connections between KQDs, Wasserstein distances \citep{Kantorovich1942,villani2009optimal}, and generalisations or approximations thereof. In particular, special cases of our KQDs recover existing sliced Wasserstein (SW) distances~\citep{bonneel2015sliced,wang_two-sample_2022,Wang2024} and can interpolate between the Wasserstein distance and MMD similarly to Sinkhorn divergences \citep{cuturi2013sinkhorn,Genevay2019}. These results are presented in Connections \ref{res:connections_slicedwasserstein}, \ref{res:connections_maxslicedwasserstein}, and \ref{res:connections_sinkhorn}.

Finally, we consider a specific instance of KQDs based on Gaussian averaging over kernelised quantile directions, which we name the \emph{Gaussian expected kernel quantile discrepancy (e-KQD)}. Beyond the desirable theoretical properties described above, we show that the Gaussian e-KQD also has attractive computational properties. In particular, we show that it has a natural estimator which only requires sampling from a Gaussian measure on the RKHS, and which can be computed with complexity $\bigo(n\log^2(n) )$. It is studied empirically in \Cref{sec:experiments} with experiments on two-sample hypothesis testing, where we show that it is competitive with the MMD: it often outperforms estimators of the MMD of the same asymptotic complexity, and in some cases even outperforms MMD at higher computational costs.

\section{Background}

Let $\calP_\calX$ denote the set of Borel probability measures on a Borel space $\calX$. We begin by reviewing existing definitions of quantiles, followed by a summary of relevant work on probability metrics, including the MMD and SW distances.

\subsection{Quantiles}
\label{sec:background_quantiles}
\textbf{Univariate quantiles.} Let $\calX \subseteq \mathbb{R}$. For $\alpha \in [0,1]$, the \emph{$\alpha$-quantile of $P \in \calP_{\calX}$} is defined as $\rho^{\alpha}_P = \inf \{y \in \calX : \text{Pr}_{Y\sim P}[Y \leq y] \geq \alpha\}$.
When $P$ has a continuous and strictly monotonic cumulative distribution function $F_P$, quantiles can also be defined through the inverse of that function $\rho^{\alpha}_P \coloneqq F^{-1}_P(\alpha)$. Notable special cases include $\alpha = 0.5$, corresponding to the median, and $\alpha = 0.25,0.75$, corresponding to lower- and upper-quartiles respectively. Importantly, $P$ is fully characterised by its quantiles $\{\rho^{\alpha}_P\}_{\alpha \in [0,1]}$.

From a computational viewpoint, univariate quantiles can be straightforwardly estimated using order statistics. Suppose $y_{1:n} = [y_1 \dots y_n]^\top \sim P$, and denote by $[y_{1:n}]_j$ the $j^\text{th}$ order statistic of $y_{1:n}$ (i.e. the $j^{\text{th}}$ largest value in the vector $[y_1 \dots y_n]^\top$). The $\alpha$-quantile of $P$, denoted $\rho^\alpha_P$, can be estimated using $[y_{1:n}]_{\lceil \alpha n \rceil}$ where $\lceil \cdot \rceil$ denotes the ceiling function. This estimator is known to converge at a rate of $\bigo(n^{-\nicefrac{1}{2}})$~\citep[Section 2.3.2]{serfling2009approximation}.

\textbf{Multivariate quantiles.} Suppose now that $\calX \subseteq \mathbb{R}^d$ for $d > 1$. The previous definition of quantiles depends on the existence of an ordering in $\calX$, and its natural generalisation to $d>1$ is therefore not unique \citep{Serfling2002}. In this paper, we will focus on the notion of \textit{$\alpha$-directional quantile of $P$ along some direction $u$ in the unit sphere $S^{d-1}$}~\citep{Kong2012},
\begin{equation*}
    \rho^{\alpha,u}_P \coloneqq \rho^{\alpha}_{\phi_u \#P} u,\qquad \phi_u(y)= \langle u, y \rangle.
\end{equation*}
Here, $\phi_u: \calX \to \Re$ is the projection map onto $u$, and $\rho^{\alpha}_{\phi_u\#P}$ is the standard one-dimensional $\alpha-$quantile of $\phi_u \# P$---the law of $\phi_u(X)$ for $X \sim P$. We note that this quantile is now a $d$-dimensional vector as opposed to a scalar.
The $\alpha$-directional quantiles for $d=2$ are illustrated in \Cref{fig:2d-quantiles}, in which the probability measure $P$ is projected onto some line; see the left and middle plots. Once again, we can use quantiles to characterise $P$, although we must now consider all $\alpha-$quantiles over a sufficiently rich family of projections $\{\rho_P^{\alpha, u} : \alpha \in [0,1], u \in S^{d-1}\}$; see Theorem 5 of \cite{Kong2012} for sufficient regularity conditions.

\begin{figure}
    \centering
    \resizebox{\columnwidth}{!}{
    \includegraphics[height=1.1in]{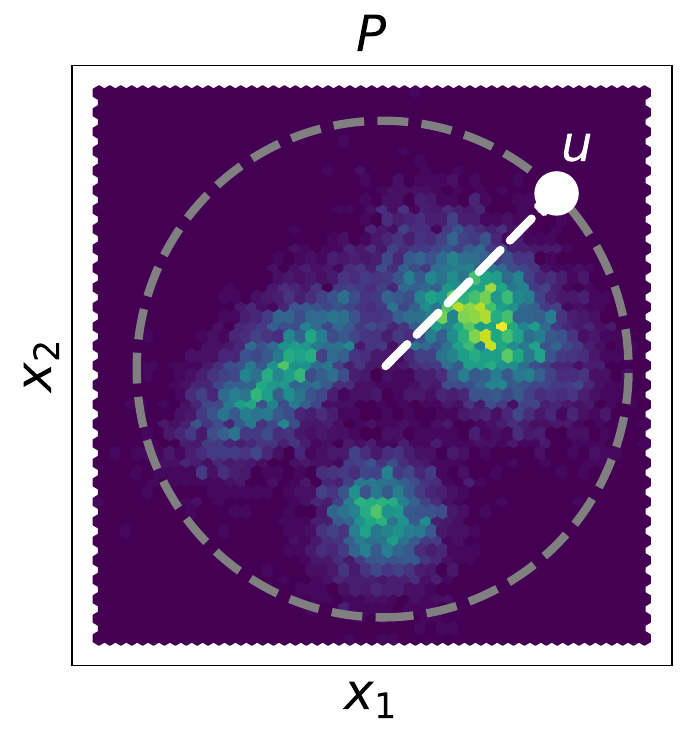}
    \includegraphics[height=1.1in]{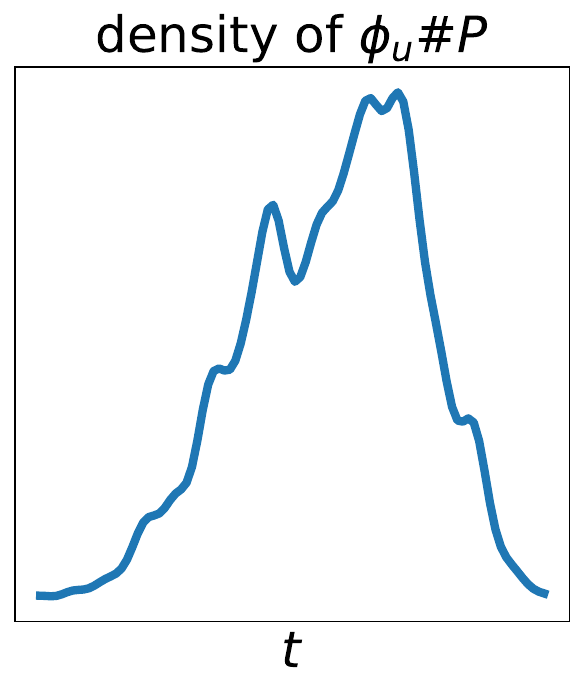}
    \includegraphics[height=1.1in]{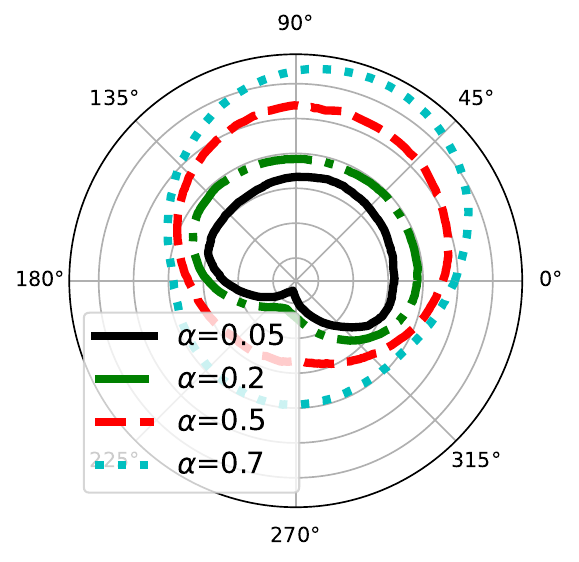}}
    \vspace{-5mm}
    \caption{\textit{Illustration of bivariate quantiles.} \textit{Left:} Bivariate distribution $P$. \textit{Center:} Density of the projection of $P$ onto direction $u$ on the unit circle, with $\phi_u(x)=\langle u, x \rangle$. \textit{Right:} different quantiles for all possible directions $u$.}
    \label{fig:2d-quantiles}
    \vspace{-5mm}
\end{figure}

Although these multivariate quantiles
satisfy scale equivariance and rotation equivariance, they do not satisfy location equivariance.
To remedy this issue,~\citet{fraiman2012quantiles} introduced a related notion, the \emph{centered} $\alpha$-directional quantile:
\begin{equation}
\label{eq:direct-q}
    \tilde \rho^{\alpha, u}_P := \left(\rho^\alpha_{\phi_u \# P} - \phi_u(\mathbb{E}_{X\sim P}[X])\right)u + \bE_{X \sim P} [X],
\end{equation}
Further details are provided in \Cref{appendix:centered_quantiles}.

\subsection{Probability Metrics}

\textbf{Kernel mean embeddings and MMD.} Let $\calX$ be some Borel space, and $(\calH, \langle \cdot, \cdot \rangle_\calH)$ be a reproducing kernel Hilbert space (RKHS) induced by a real-valued kernel $k: \calX \times \calX \to \Re$ \citep{Scholkopf01:LKS,Berlinet2004},
the \emph{kernel mean embedding} (KME) $\mu_P: \calX \to \Re$ of any $P \in \calP_\calX$ is defined as the Bochner integral $\mu_P(\cdot)=\bE_{X\sim P}[k(X, \cdot)] \in \calH$.
The integral can be shown to exist provided $\bE_{X \sim P}[\sqrt{k(X, X)}]<\infty$; this, in turn, holds for all $P \in \calP_\calX$ if and only if $k$ is bounded \citep{Smola07Hilbert}. If the mapping $P \to \mu_P$ is injective, the kernel $k$ is said to be \emph{mean-characteristic}. Many standard kernels---the Mat\'ern family, Gaussian, Laplacian---have been shown to be characteristic on sufficiently regular spaces~\citep{sriperumbudur2011universality,ziegel2022characteristic}.
KMEs with mean-characteristic kernels lead to the squared \textit{maximum mean discrepancy (MMD)} defined for any $P,Q \in \calP_\calX$ as
\begin{align*}
    &\text{MMD}^2(P,Q) \coloneqq \|\mu_P - \mu_Q \|_{\mathcal{H}}^2 = \bE_{X,X' \sim P}[k(X,X')] \\
    &\hspace{1cm}
    - 2 \bE_{X \sim P,X' \sim Q}[k(X,X')] + \bE_{X,X' \sim Q}[k(X,X')].
\end{align*}
This can be computed in rare cases \citep{Briol2025}, but typically needs to be estimated.
Given $n$ i.i.d. realisations from $P$ and $Q$, MMD$^2$ is most commonly estimated with a U-statistic, which converges to $\text{MMD}^2(P,Q)$ as $\bigo(n^{-\nicefrac{1}{2}})$ and has computational complexity of $\bigo(n^2)$---although linear-cost alternatives are also available~\citep[Lemma 14.]{gretton2012kernel}. These are discussed in \Cref{sec:experiments} and \Cref{appendix_sec:estimators}.

\textbf{Wasserstein distances.} Let $c: \calX \times \calX \to \Re$ be a metric on $\calX$, and $\Gamma(P,Q) \subseteq \calP_{\calX \times \calX}$ denote the space of joint distributions on $\cal{X} \times \cal{X}$ with first and second marginals $P$ and $Q$, respectively. The \textit{$p$-Wasserstein distance}~\citep{Kantorovich1942,villani2009optimal} quantifies the cost of optimally transporting one distribution to another under ``cost'' $c: \calX \times \calX \rightarrow \Re$. It is a probability metric under mild conditions~\citep[Section 6]{villani2009optimal}, and is defined as
\begin{align*}
    W_p(P,Q) = \left(\inf_{\pi \in \Gamma(P,Q)} \bE_{(X,Y)\sim \pi}\left[c(X,Y)^p \right]\right)^{\nicefrac{1}{p}}.
\end{align*}
When $\calX \subseteq \mathbb{R}^d$, the metric $c$ is typically taken to be the Euclidean distance $c(x,y)=\|x-y\|_2$.
The Wasserstein distance can then be estimated by solving an optimal transport problem using empirical measures constructed through samples of $P$ and $Q$, an approach that suffers from a high computational cost of $\bigo(n^3)$ and, when $P,Q$ have at least $2p$ moments, slow convergence of $\bigo(n^{-1 / \max(d, 2p)})$ when $\calX \subseteq \Re^d$ for $d>1$ \citep{Fournier2015}.

However, when $d=1$, $W_p$ can be computed at lower cost of $\bigo(n \log n)$ with convergence of $\bigo(n^{-\nicefrac{1}{2p}})$ when $P,Q$ have at least $2p$ moments.
This motivated the introduction of the \textit{sliced Wasserstein} (SW) distance~\citep{bonneel2015sliced}. Recall that $\phi_u(x) = u^\top x$. The SW distance projects high-dimensional distributions $P, Q$ onto elements on the unit sphere $u \in S^{d-1}$ sampled uniformly, computes the Wasserstein distance between the projected distributions, now in $\mathbb{R}$, and averages over the projections:
\begin{equation*}
    \text{SW}_p(P, Q) =\left( \bE_{u \sim \bU(S^{d-1})}\left[W_p^p(\phi_u\#P, \phi_u\#Q)\right] \right)^{\nicefrac{1}{p}}.
\end{equation*}
A further refinement, the \textit{max-sliced Wasserstein (max-SW)} distance~\citep{Deshpande2018}, aims to identify the optimal projection that maximises the 1D Wasserstein distance:
\begin{equation*}
    \text{max-SW}_p(P, Q) =\left( \sup_{u \in S^{d-1}} W_p^p(\phi_u\#P, \phi_u\#Q) \right)^{\nicefrac{1}{p}}.
\end{equation*}
Both slicing distances reduce the computational complexity to $\bigo(l n \log n)$ and the convergence rate to $\bigo(l^{-\nicefrac{1}{2}} + n^{-\nicefrac{1}{2p}} )$, where $l$ is either the number of projections, or the number of iterations of the optimiser. A further extension is the \textit{Generalised Sliced Wasserstein} (GSW,~\citet{Kolouri2019}), which replaces the linear projection $\phi_u$ with a non-linear mapping. While the conditions for GSW to be a probability metric are highly non-trivial to verify, the authors showed that they hold for polynomials of odd degree.

Another approximation of the Wasserstein distance involves the introduction of an entropic regularisation term~\citep{cuturi2013sinkhorn}, which reduces the cost to $\bigo(n^2)$ and can be estimated with sample complexity $\bigo(n^{-\nicefrac{1}{2}})$ \citep{Genevay2019}. The solution to this regularised problem is referred to as the \emph{Sinkhorn divergence}. Interestingly,~\citet{ramdas2017wasserstein,Feydy2019} demonstrated that by varying the strength of the regularisation, the Sinkhorn divergence interpolates between the Wasserstein distance and the MMD with a kernel corresponding to the energy distance.

\section{Kernel Quantile Embeddings and Discrepancies}
\label{sec:KQE_KQD}

We introduce directional quantiles in the RKHS and the corresponding discrepancies. Unlike in~\Cref{sec:background_quantiles}, the measures and their quantiles now live in different spaces: the measures are on $\calX$, and the quantiles are in the RKHS $\calH$ induced by a kernel on $\calX$.
This leads to greater flexibility: the approach works for any space a kernel can be defined on. Throughout, we assume the kernel $k$ is measurable.

\subsection{Kernel Quantile Embeddings}

Let $S_\calH = \{u \in \calH : \|u\|_\calH = 1\}$ be the unit sphere of an RKHS $\calH$ induced by the kernel $k$. For $P \in \calP_\calX$, we define its \textit{$\alpha$-quantile along RKHS direction $u \in S_\calH$} as a function $\rho_P^{\alpha, u}:\calX \rightarrow \mathbb{R}$ in $ \calH$ with
\begin{equation}
\label{eq:quemb_for_kernels}
    \rho_P^{\alpha,u}(x) \coloneqq \rho^\alpha_{u \# P} u(x)
\end{equation}
By the reproducing property, it holds that $\rho^\alpha_{u \# P} u(x) = \rho^\alpha_{\phi_u \# [\psi \# P]} u(x)$, where $\psi(x) = k(x, \cdot)$ is the canonical feature map $\calX \to \calH$, and $\phi_u(h) = \langle u, h\rangle_{\calH}$ is the $\calH \to \calH$ equivalent of the projection operator onto $u$ defined in~\Cref{sec:background_quantiles}. Thus, when $\dim(\calH)<\infty$, the RKHS quantiles of $P$ on $\calX$ are exactly the multivariate quantiles of the measure of $k(X, \cdot)$, $X \sim P$, on $\calH$.
In other words, KQEs can be thought of as two-step embeddings: we first embed $X \sim P \in \calP_\calX$ as an RKHS element and then compute its directional quantiles to obtain the KQEs.

\textbf{Centered vs uncentered quantiles.}
Just as done for multivariate quantiles in~\Cref{eq:direct-q}, a centered version of RKHS quantiles can be defined as
\begin{equation*}
    \tilde \rho_P^{\alpha,u}(x) \coloneqq \left(\rho^\alpha_{u \# P} - \langle u, \mu_P \rangle_\calH \right)u(x) + \mu_P(x), 
\end{equation*}
where $\mu_P$ is the KME of $P$. This coincides with~\Cref{eq:direct-q} for the measure being
the law of $k(X, \cdot)$ with $X \sim P$. The impact of centering is examined in detail in~\Cref{appendix:centered_quantiles}, but two key observations are relevant here: (1) omitting centering eliminates the computational overhead of calculating means; (2)
the only equivariance violated for the uncentered directional quantile is location equivariance: shifting $k(X, \cdot)$ by $h$ shifts the quantile by $\langle h, u \rangle_\calH u$, rather than by $h$ itself.
However, when KQEs are used to compare two distributions, the additional term $\langle h, u \rangle_\calH u$ cancels out as it does not depend on the measure. For these reasons, we primarily work with the uncentered RKHS quantiles.

\textbf{Quantile-characteristic kernels.} The kernel $k$ is said to be \emph{quantile-characteristic} if the mapping $P \mapsto \{\rho_P^{\alpha,u} : \alpha \in [0, 1], u \in S_\calH\}$ is injective for $P \in \calP_\calX$. In $\Re^d$, the Cram\'er-Wold theorem~\citep{cramer1936some} states that the set of all one-dimensional projections (or, equivalently, all quantiles of all one-dimensional projections) determines the measure. One may therefore recognise our next theorem as an RKHS-specific extension of the Cram\'er-Wold theorem. Earlier Hilbert space extensions required higher-dimensional projections and imposed restrictive moment assumptions~\citep{cuesta2007sharp}. Being concerned with the RKHS case specifically allows us to prove the result under mild assumptions, as stated below.
\begin{assumption}
\label{as:input_space}
    \hspace{-0.1cm}$\calX$ is Hausdorff, separable, and $\sigma$-compact.
\end{assumption}
Being Hausdorff ensures points in $\calX$ can be separated, and separability says $\calX$ has a countable dense subset. $\sigma$-compactness means $\calX$ is a union of countably many compact sets. These are mild conditions, notably satisfied by Polish spaces---including discrete topological spaces with at most countably many elements and topological manifolds. 

It is possible to drop the $\sigma$-compactness and separability. When $\calX$ is Hausdorff and completely regular, one can still get quantile-characteristic properties on Radon probability measures---the "non-pathological" Borel probability measures. We discuss this in~\Cref{sec:proof_projections_determine_distribution} and refer to~\citet{Willard1970} for a review of general topological properties.
\begin{assumption}
\label{as:kernel}
    The kernel $k$ is continuous, and separating on $\calX$: for any $x \neq y \in \calX$, it holds that $k(x, \cdot) \neq k(y, \cdot)$.
\end{assumption}
This is a mild condition: most commonly used kernels such as the Mat\'ern, Gaussian, and Laplacian kernels are separating. The constant kernel $k(x, x')=c$ is an example of a non-separating kernel. Trivially, a non-separating kernel for which $k(x, \cdot) = k(y, \cdot)$ will not be able to distinguish between Dirac measures $\delta_x$ and $\delta_y$.
The proof of the following result uses~\emph{characteristic functionals}, an extension of characteristic functions to measures on spaces beyond $\mathbb{R}^d$. Unlike moments, characteristic functionals are defined for any probability measure---which is the key to generality of KQEs. Further discussion and proof are in~\Cref{sec:proof_projections_determine_distribution}.
\begin{theorem}[\textbf{Cram\'er-Wold Theorem in RKHS}]
\label{res:cramer-wold}
    Under~\Cref{as:input_space,as:kernel}, the kernel $k$ is quantile-characteristic, meaning the mapping $P \mapsto \{\rho_P^{\alpha,u} : \alpha \in [0, 1], u \in S_\calH\}$ is injective.
\end{theorem}

The mildness of the assumptions in Theorem~\ref{res:cramer-wold} naturally raises the question: \emph{is being quantile-characteristic a less restrictive condition than being mean-characteristic?}
This indeed holds, as shown in the result below.
\begin{theorem}
\label{res:if_meanchar_then_quantchar}
    Every mean-characteristic kernel $k$ is also quantile-characteristic. The converse does not hold.
\end{theorem}
This result, proven in \Cref{appendix:proof_if_meanchar_then_quantchar}, has a powerful implication.
For any discrepancy $D(P, Q)$ that aggregates the KQEs injectively (i.e $D(P, Q)=0 \iff \rho_P^{\alpha,u}=\rho_Q^{\alpha,u}$ for all $\alpha, u$), it holds that $\MMD(P, Q) > 0 \Rightarrow D(P, Q) > 0$, but $D(P, Q) > 0 \not\Rightarrow \MMD(P, Q) > 0$. This means $D$ can tell apart every pair of measures MMD can, and sometimes more (see the proof for examples). This is intuitive: MMD is an injective aggregation of means ($\MMD(P, Q)=0 \iff \bE_P[u]=\bE_Q[u]$ for all $u$), and the set of all quantiles captures all the information in the mean, but not vice versa.
Before introducing a specific family of quantile discrepancies, we discuss sample versions of KQEs.
\begin{figure}
    \centering
    \resizebox{\columnwidth}{!}{
    \includegraphics[height=1.1in]{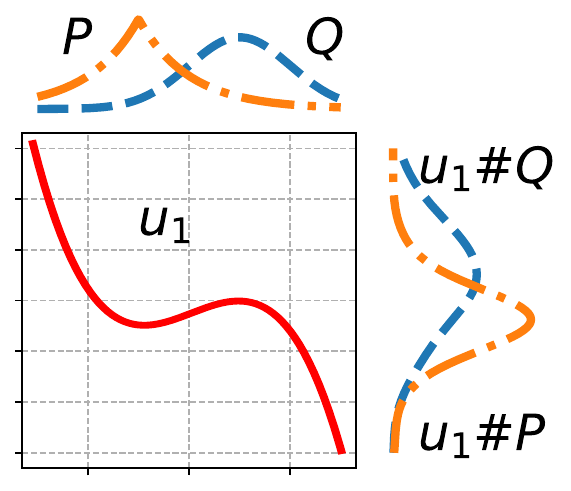}
    \includegraphics[height=1.1in]{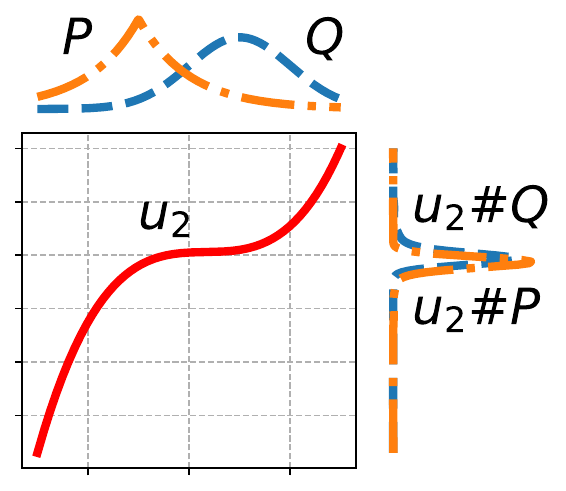}
    \includegraphics[height=1.1in]{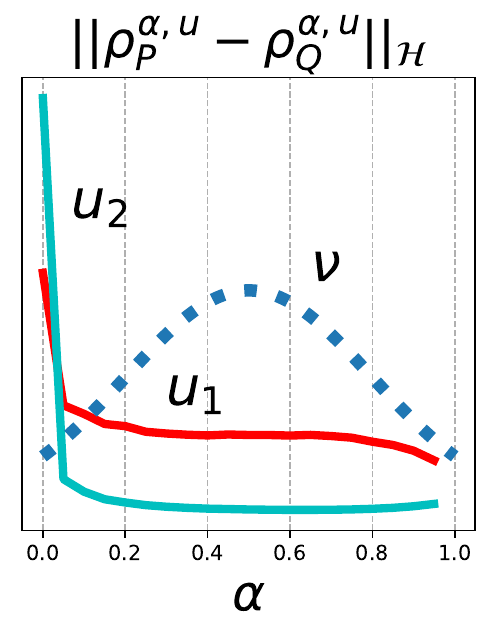}}
    \vspace{-5mm}
    \caption{\textit{Illustration of the impact of the slicing direction on KQEs.} Suppose $X \sim P$, the KQEs $\rho_P^{\alpha, u}(x) \coloneqq \rho_{u \#P}^{\alpha} u(x)$ are obtained by considering the $\alpha^\text{th}$ quantile of $u(X)$. Clearly, these quantiles might vary significantly depending on the slicing direction used.
    }
    \label{fig:kernelised-quantiles}
    \vspace{-5mm}
\end{figure}

\textbf{Estimating KQEs.}
For fixed $\alpha \in [0, 1]$ and $u \in S_\calH$, estimating the directional quantile $\rho_P^{\alpha,u}$ with samples $x_{1:n} \sim X$ boils down to estimating the $\Re$-quantile $\rho_{u \# P}^{\alpha}$ using samples $u(x_{1:n})$. We employ the classic, model-free approach to estimate a quantile by using the order statistic estimator:
\begin{equation}
\label{eq:dq_estimator}
    \rho^{\alpha,u}_{P_n}(x) := \rho_{u \# P_n}^{\alpha}u(x) = [ u(x_{1:n}) ]_{\lceil \alpha n \rceil}u(x),
\end{equation}
where $P_n = \nicefrac{1}{n}\sum_{i=1}^n \delta_{x_i}$. In other words, \Cref{eq:dq_estimator} uses the $\alpha$-quantile of the set $u(x_{1:n})$---meaning, the $\lceil \alpha n \rceil$-th largest element of $u(x_{1:n})$. We now state an RKHS version of a classic result on convergence of quantile estimators; the proof is provided in~\Cref{appendix:proof_consistency_KQE}.

\begin{theorem}[\textbf{Finite-Sample Consistency for Empirical KQEs}]\label{res:consistency_KQE}
    Suppose the PDF of $u\#P$ is bounded away from zero, $f_{u\#P}(x) \geq c_u > 0$, and $x_{1:n} \sim P$. Then, with probability at least $1-\delta$, and $C(\delta, u) = \bigo(\sqrt {\log(2/\delta)})$,
    \begin{equation*}
        \| \rho^{\alpha, u}_{P_n} - \rho^{\alpha, u}_{P} \|_\calH \leq C(\delta, u) n^{-\nicefrac{1}{2}}.
    \end{equation*}
\end{theorem}
We do not need to assume~\Cref{as:input_space,as:kernel} to prove consistency; this was only needed to establish that $k$ is quantile-characteristic, and we may still have a consistent estimator when the kernel is not quantile-characteristic. The condition $f_{u\#P}(x) \geq c_u > 0$ lets us avoid making any assumptions on $\calX$, other than the existence of a kernel $k$ on $\calX$.

\subsection{Kernel Quantile Discrepancies}

We propose to quantify the difference between $P,Q \in \calP_\calX$ in unit-norm direction $u$ through a $\nu$-weighted expectation of power-$p$ distance (in the RKHS) between KQEs,
\begin{align*}
    \tau_p(P, Q;\nu, u) = \Big(\int_0^1 \big\| \rho_P^{\alpha,u} - \rho_Q^{\alpha,u} \big\|_\calH^p \nu(\d \alpha) \Big)^{\nicefrac{1}{p}}.
\end{align*}
\Cref{fig:kernelised-quantiles} illustrates how $u\#P$ and $u\#Q$ vary depending on direction $u$---and the impact it has on $\tau_p$. The weighting measure $\nu$ on $[0, 1]$ assigns importance to each $\alpha$-quantile. For example, the Lebesgue measure $\nu \equiv \mu$ treats all quantiles as equally important, whereas a partial-supported measure would allow us to ignore certain quantiles.

Based on $\tau_p(P,Q; \nu, u)$, we introduce a novel family of \textit{Kernel Quantile Discrepancies (KQDs)} that aggregate the directional differences $\tau_p(P, Q; \nu, u)$ over $u \in S_\calH$: the $L^p$-type distance \textit{expected KQD ($\ekqd$)} that uses the average as the aggregate function, and the $L^\infty$-type distance \textit{supremum KQD ($\supkqd$)} that aggregates with the supremum:
\begin{equation}
\label{eq:general_distances}
\begin{split}
    \ekqd_p(P, Q; \nu, \gamma) &= \left(\bE_{u \sim \gamma} \left[\tau^p_p\left(P, Q; \nu, u\right)\right]\right)^{\nicefrac{1}{p}},\\
    \supkqd_p(P, Q; \nu) &= \big(\sup_{u \in S_\calH} \tau^p_p\left(P,Q; \nu,u \right)\big)^{\nicefrac{1}{p}},
\end{split}
\end{equation}
where $\gamma$ is a measure on the unit sphere $S_\calH$ of the RKHS.

Next, we demonstrate that under mild conditions $\ekqd$ and $\supkqd$ are indeed distances, and establish connections with existing methods. The proof is in \Cref{appendix:proof_quantile_characteristic}.
\begin{theorem}[\textbf{KQDs as Probability Metrics}]\label{res:KQE_characterise_dists}
    Under~\Cref{as:input_space},~\Cref{as:kernel}, and if $\nu$ has full support on $[0, 1]$, $\supkqd_p$ is a distance. Further, if $\gamma$ has full support on $S_\calH$, $\ekqd_p$ is a distance.
\end{theorem}
As discussed in~\Cref{sec:KQE_KQD}, ~\Cref{as:input_space,as:kernel} are minor. The assumptions on the support of $\nu$ and $\gamma$ ensure that no quantile level in $[0,1]$ and no parts of $S_\calH$ are missed entirely. This is satisfied, for example, for the uniform $\nu$ (that considers all quantiles to be equally important), and when $\calH$ is separable, for any centered Gaussian $\gamma=\calN(0, S)$ with a non-degenerate $S$ by~\citep[Corollary 5.3]{kukush2020gaussian}. For example, an $\calH \mapsto \calH$ covariance operator $S[f](x) = \int_\calX k(x, y) f(y) \beta(\d y)$ is non-degenerate and well-defined provided (1) $\beta$ on $\calX$ has full support, and (2) $\int_\calX \sqrt {k(x, x)} \beta(\d x) < \infty$. This choice of $\gamma$ also happens to be computationally convenient, as discussed in~\Cref{sec:estimator}.

In contrast, while conditions under which MMD is a distance are well-understood for bounded translation-invariant kernels on Euclidean spaces~\citep{sriperumbudur2011universality}, they are challenging to establish beyond this setting. For instance, it is known that commonly used graph kernels are not characteristic~\citep{kriege2020survey}.

When $\nu$ is chosen as the Lebesgue measure $\mu$, an important connection emerges between $\ekqd$, $\supkqd$, and sliced Wasserstein distances. This connection is formalised in the next result, with a proof provided in \Cref{appendix:proof_connections_slicedwasserstein}.

\begin{connection}[\textbf{SW}]
\label{res:connections_slicedwasserstein}
Suppose $P, Q$ have $p$-finite moments. Then, $\ekqd_p(P, Q; \nu, \gamma)$ for $\nu \equiv \mu$ corresponds to a kernel expected sliced $p$-Wasserstein distance, which has not been introduced in the literature. For $\calX \subseteq \mathbb{R}^d$, linear $k(x, y)=x^\top y$, and uniform $\gamma$, this recovers the expected sliced $p$-Wasserstein distance~\citep{bonneel2015sliced}.
\end{connection}
\begin{connection}[\textbf{Max-SW}]
\label{res:connections_maxslicedwasserstein}
Suppose $P, Q$ have $p$-finite moments. Then, $\supkqd_p(P, Q; \nu)$ for $\nu \equiv \mu$ is the kernel max-sliced $p$-Wasserstein distance~\citep{wang_two-sample_2022}. For $\calX \subseteq \mathbb{R}^d$, linear $k(x, y)=x^\top y$, and uniform $\gamma$, it recovers the max-sliced $p$-Wasserstein~\citep{Deshpande2018}.
\end{connection}

For $d=1$, we recover standard Wasserstein. When $k$ is non-linear but induces a finite-dimensional RKHS, $\ekqd$ is connected to the Generalised Sliced Wasserstein distances of~\citet{kolouri_generalized_2022}---we explore this in~\Cref{appendix:proof_connections_slicedwasserstein}.

Lastly, we establish a connection to Sinkhorn divergence.

\begin{connection}[\textbf{Sinkhorn}]
\label{res:connections_sinkhorn}
    Sinkhorn divergence~\citep{cuturi2013sinkhorn}, like $\ekqd$ and $\supkqd$, combines the strengths of kernel embeddings and Wasserstein distances. Furthermore, for $p=2$ and $\nu \equiv \mu$, the centered version of $\ekqd$ and $\supkqd$ developed in~\Cref{appendix:centered_quantiles} can be represented as a sum of MMD and kernelised expected or max-sliced Wasserstein distances, thus positioning these measures as mid-point interpolants between MMD and SW distances.
\end{connection}
It is important to note that the MMD term within the Sinkhorn divergence is restricted to a specific kernel tied to the energy distance---in contrast, $\ekqd$ and $\supkqd$ offer much greater flexibility in the choice of kernel. Moreover, as will be shown empirically in~\Cref{sec:experiments}, the computational complexity of $\ekqd$ for a particular choice of $\gamma$ can be made significantly lower than that of Sinkhorn divergences, which have a cost of $\bigo(n^2)$.

\paragraph{Estimating $\ekqd$.} We propose a Monte-Carlo estimator for $\ekqd$, and refer to~\citet{wang_two-sample_2022} for an optimisation-based, $\bigo(n^3 \log(n))$ estimator for $\supkqd$.
Let $x_{1:n} \sim P$, $y_{1:n} \sim Q$, the $u_1, \dots, u_l \in S_\calH$ to be $l$ unit-norm functions sampled from $\gamma$, and $f_\nu$ to be the density of $\nu$. Denote $P_n = \nicefrac{1}{n}\sum_{i=1}^n \delta_{x_i}, Q_n = \nicefrac{1}{n} \sum_{i=1}^n \delta_{y_i}$. Then, similarly to the order statistic estimator of the quantiles in~\Cref{eq:dq_estimator}, $\ekqd^p_p(P_n, Q_n; \nu, \gamma_l)$ is the estimator of $\ekqd^p_p(P, Q; \nu, \gamma)$, where
\begin{align}
\label{eq:estimator_ekqd}
    &\ekqd^p_p(P_n, Q_n; \nu, \gamma_l) \\
    &\hspace{0.2cm}= \frac{1}{ln} \sum_{i=1}^l \sum_{j=1}^{n} \left(\big[u_i(x_{1:n})\big]_j - \big[u_i(y_{1:n})\big]_j \right)^p f_\nu\left(\left\lceil \nicefrac{j}{n} \right\rceil \right) \nonumber
\end{align}
Here, $[u_i(x_{1:n})]_j$ is the $j$-th order statistics, meaning the $j$-th smallest element of $u_i(x_{1:n})=[u_i(x_1), \dots, u_i(x_n)]^\top$. For $p=1$, we get the following result, proven in~\Cref{appendix:proof_consistency_KQD}.
\begin{theorem}[\textbf{Finite-Sample Consistency for Empirical KQDs}]
\label{res:consistency_KQD}
    Let $\nu$ have a density, $P, Q$ be measures on $\calX$ s.t.
    $\bE_{X \sim P} \sqrt{k(X, X)}<\infty$ and $\bE_{X \sim Q} \sqrt{k(X, X)}<\infty$, and $x_{1:n} \sim P, y_{1:n} \sim Q$. Then, with probability at least $1-\delta$, and $C(\delta) = \bigo(\sqrt{\log(1/\delta)})$ that depends only on $\delta, k, \nu$,
    \begin{align*}
        &| \ekqd_1(P_n, Q_n;\nu, \gamma_l) - \ekqd_1(P, Q;\nu, \gamma) | \\
        &\hspace{4cm}\leq C(\delta)(l^{-\nicefrac{1}{2}} + n^{-\nicefrac{1}{2}}).
    \end{align*}
\end{theorem}
The rate does not depend on $\mathrm{dim}(\calX)$---this is a major advantage of projection/slicing-based discrepancies~\citep{Nadjahi2020}, which comes at the cost of dependence on the number or projections $l$. Setting $l=n/\log n$ recovers the MMD rate (up to log-terms), at matching complexity (see~\Cref{sec:estimator}). Here, we do not need $\ekqd$ to be a distance---indeed, we did not assume~\Cref{as:input_space,as:kernel}.
The condition of square root integrability of $k(X, X)$ under $P, Q$ is immediately satisfied when $k$ is bounded, and can in fact be further weakened to $\bE_{X \sim P} \bE_{Y \sim Q} \sqrt{k(X, X) - 2 k(X, Y) + k(Y, Y)} < \infty$.
Requiring that $\nu$ has a density is mild and necessary to reduce the problem to CDF convergence---which, by the classic Dvoretzky-Kiefer-Wolfowitz inequality of~\citet{dvoretzky1956asymptotic} has rate $n^{-1/2}$ under no assumptions on the underlying distributions. The strength of this inequality allows us to assume nothing more of $\calX$ than the fact that it is possible to define a kernel on it.

Further, for any integer $p>1$, the $n^{-1/2}$ rate still holds---if and only if it holds that for $J_p(R) \coloneqq \left(F_{u\#R}(t) (1 - F_{u\#R}(t))\right)^{p/2} / f^{p-1}_{u\#R}(t)$, both $J_p(P)$ and $J_p(Q)$ are integrable over $u \sim \gamma$ and Lebesgue measure on $u(\calX)$.
In turn, this may be reduced to a problem of controlling $d-1$ volumes of level sets of $u$. We discuss this extension further in~\Cref{res:consistency_KQD_p} in~\Cref{appendix:proof_consistency_KQD}.

\section{Gaussian Kernel Quantile Discrepancy}\label{sec:estimator}

\begin{algorithm}[tb]
    \caption{Gaussian $\ekqd$}
    \label{alg:gaussian_ekqd}
\begin{algorithmic}
    \STATE {\bfseries Input:} Data $x_{1:n} \sim P, y_{1:n} \sim Q$, samples from the reference measure $z_{1:m} \sim \xi$, kernel $k$, density $f_\nu$, number of projections $l$, power $p$.
    \STATE Initialise $\ekqd^p \leftarrow 0$ and $\tau_{p,i}^p \leftarrow 0$ for $i=1\dots l$.
    \FOR{$i=1$ {\bfseries to} $l$}
    \STATE Sample $\lambda_{1:m} \sim \calN(0, \Id_m)$
    \STATE Compute $f_i(x_{1:n}) \leftarrow \lambda_{1:m}^\top k(z_{1:m}, x_{1:n})/\sqrt{m},$
    \STATE \hspace{1.3cm} $f_i(y_{1:n}) \leftarrow \lambda_{1:m}^\top k(z_{1:m}, y_{1:n}) /\sqrt{m}$
    \STATE Compute $\|f_i\|_\calH \leftarrow \sqrt{\lambda_{1:m}^\top k(z_{1:m}, z_{1:m}) \lambda_{1:m}/m}$
    \STATE Compute $u_i(x_{1:n}) \leftarrow f_i(x_{1:n})/\|f_i\|_\calH$, \\ \hspace{1.3cm} $u_i(y_{1:n}) \leftarrow f_i(y_{1:n})/\|f_i\|_\calH$
    \STATE Sort $u_i(x_{1:n})$ and $u_i(y_{1:n})$
    \FOR{$j=1$ {\bfseries to} $n$}
    \STATE $\tau_{p,i}^p \gets \tau_{p, i}^p + \big([u_i(x_{1:n})]_j - [u_i(y_{1:n})]_j\big)^p f_\nu(\lceil j/n \rceil)$
    \ENDFOR
    \STATE $\ekqd^p \gets \ekqd^p + \tau_{p,i}^p/l$
    \ENDFOR
    \STATE Return $\ekqd^p$
\end{algorithmic}
\end{algorithm}

We now conduct further empirical study of the squared kernel distance $\ekqd_p$. Unlike its supremum-based counterpart $\supkqd$, $\ekqd$ can be approximated simply by drawing samples from $\gamma$ on $S_\calH$, avoiding the challenges associated with optimising for the supremum. Although a uniform $\gamma$ is a natural choice, no such measure exists when $\dim(\calH)$ is infinite~\citep[Section 1.3]{kukush2020gaussian}. Instead, we follow a well-established strategy from the inverse problems literature~\citep{stuart2010inverse} and take $\gamma$ to be the projection onto $S_\calH$ of a Gaussian measure on $\calH$. Using established techniques for sampling Gaussian measures, we then build an efficient estimator for $\ekqd_p(P,Q;\nu,\gamma)$. Gaussian measures on Hilbert spaces are a natural extension of the familiar Gaussian measures on $\Re^d$: a measure $\calN(0, C)$ on $\calH$
is said to be a \emph{centered Gaussian measure} with covariance operator $C:\calH \to \calH$ if, for every $f \in \calH$, the pushforward of $\calN(0, C)$ under the $\calH \to \Re$ projection map $\phi_f(\cdot)= \langle f, \cdot \rangle_\calH$ is the Gaussian measure $\calN(0, \langle C[f], f \rangle_\calH)$ on $\Re$. For further details on Gaussian measures in Hilbert spaces, we refer to \citet{kukush2020gaussian}.

Let $\gamma'$ be a centered Gaussian measure on $\calH$ whose covariance function $C: \calH \to \calH$ is an integral operator with some reference measure $\xi$ on $\calX$,
\begin{align*}
    \gamma' = \calN(0, C),\quad C[f](x) = \int_\calX k(x, y) f(y) \xi(\d y),
\end{align*}
and let $\gamma$ be the pushforward of $\gamma'$ by the projection $\calH \to S_\calH$ that maps any $f \in \calH$ to $f/\|f\|_\calH \in S_\calH$. By the change of variables formula for pushforward measures~\citep[Theorem 3.6.1]{bogachev2007measure}, it holds that
\begin{align*}
    \ekqd_p^p(P, Q; \nu, \gamma)
    &= \bE_{u \sim \gamma}\left[ \tau_p^p\left(P, Q; \nu, u \right) \right]\\
    &= \bE_{f \sim \gamma'} \left[ \tau_p^p\left(P, Q; \nu, \nicefrac{f}{\|f\|_\calH}\right)\right].
\end{align*}
This equality reduces sampling from $\gamma$ to sampling from a centered Gaussian measure with an integral operator covariance function. The next proposition reduces sampling from (a finite-sample approximation of) $\gamma$ to sampling from the standard Gaussian on the real line; proof is in~\Cref{sec:proof_sampling_from_gm}.
\begin{proposition}[\textbf{Sampling from a Gaussian measure}]
\label{res:sampling_from_gm}
    Let $z_{1:m} \sim \xi$, and $\gamma'_m$ to be the estimate of $\gamma'$ based on the Monte Carlo estimate $C_m$ of the covariance operator $C$,
    \begin{align*}
        \gamma'_m = \calN(0, C_m),\quad C_m[g](x) = \frac{1}{m}\sum_{j=1}^m k(x, z_j) g(z_j).
    \end{align*}
    Suppose $f(x) = m^{-1/2}\sum_{j=1}^m \lambda_j k(x, z_j)$ with $\lambda_{1:m} \sim \calN(0, 1)$. Then, $f \sim \gamma'_m$.
\end{proposition}
\Cref{alg:gaussian_ekqd} brings together the $\ekqd$ estimator in~\Cref{eq:estimator_ekqd}, and the procedure for sampling from the Gaussian measure in~\Cref{res:sampling_from_gm}. The $\nu$ choice is left up to the user; the uniform $\nu$ remains a default choice. We proceed to analyse the cost.
This estimator has complexity $ \bigo(l \max(nm, m^2, n \log n))$:
$\bigo(l)$ for iterating over directions $i \in \{1, \dots, l\}$;
$\bigo(nm)$ for computing $f_i(x_{1:n})$ and $f_i(y_{1:n})$;
$\bigo(m^2)$ for computing $\| f_i \|_\calH$; and
$\bigo(n \log n)$ for sorting $u_i(x_{1:n})$ and $u_i(y_{1:n})$. For $l \coloneqq \log n$ and $m \coloneqq \log n$, the complexity therefore reduces to $ \bigo(n \log^2 n)$; i.e. near-linear (up to log-terms).

\section{Experiments}
\label{sec:experiments}

\begin{figure*}[t]
    \centering
    \includegraphics[width=\linewidth]{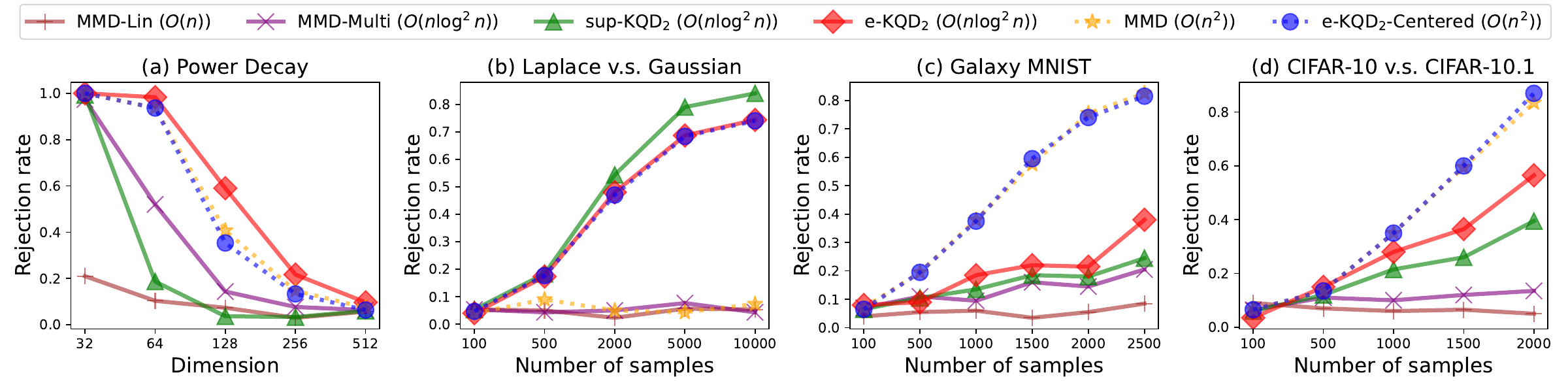}
    \caption{Experimental results comparing our proposed methods with baseline approaches. Methods represented by dotted lines exhibit quadratic complexity for a \textbf{single} computation of the test statistic, while the remaining methods achieve near-linear or linear computational efficiency. A higher rejection rate indicates better performance in distinguishing between distributions. \textbf{Overall, quadratic-time quantile-based estimators perform comparably to quadratic-time MMD estimators, while near-linear time quantile-based estimators often outperform their MMD-based counterparts.}}
    \label{fig: main_experiments}
\end{figure*}

We empirically demonstrate the effectiveness of KQDs for nonparametric two-sample hypothesis testing which aims at determining whether two arbitrary probability distributions, $P$ and $Q$, differ statistically based on their respective i.i.d. samples. Two-sample testing is widely adopted in scientific discovery fields, such as model verification~\citep{gao2024model}, out-of-domain detection~\citep{magesh2023principled}, and comparing epistemic uncertainties~\citep{chau2025credal}. Specifically, we test the null hypothesis $H_0: P = Q$ against the alternative $H_1: P \neq Q$. In such tests, (estimators of) probability metrics are commonly used as test statistics, including the Kolmogorov-Smirnov distances~\citep{Kolmogorov60:Foundations}, Wasserstein distance~\citep{wang_two-sample_2022}, energy-distances~\citep{SzeRiz05,sejdinovic2013equivalence}, and most relevant to our work, the MMD~\citep{gretton2006kernel,gretton2009fast, gretton2012kernel}. For an excellent overview of kernel-based two sample testing, we refer readers to \citet{schrab2025unified}.

Experiments are repeated to calculate the rejection rate, which is the proportion of tests where the null hypothesis is rejected.
A high rejection rate indicates better performance at distinguishing between distributions.
It is equally important to ensure proper control of Type I error, defined as the rejection rate when the null hypothesis $H_0$ is true. Specifically, the Type I error rate should not exceed the specified level. Without controlling for Type I error, an inflated rejection rate might not reflect the estimator’s ability to detect genuine differences but instead indicate the test rejects more often than it should. We consider a significance level $\alpha$ of 0.05 throughout and report on Type I control in~\Cref{appendix:experiments}.

To determine the rejection threshold for each test statistic, we employ a permutation-based approach: for each trial we pool the two samples, randomly reassign labels $300$ times to simulate draws under $H_0$, compute the test statistic on each permuted split, and take the 95th percentile of this empirical null distribution as our threshold. This fully nonparametric thresholding ensures Type I error control without additional distributional assumptions~\citep{lehmann1986testing}.

Our experiments aim to demonstrate that, within a comparable computational budget, statistics computed using quantile-characteristic kernels can deliver results competitive with those of MMD tests based on mean-characteristic kernels. Additionally, we seek to explore the inherent trade-offs of the proposed methods. We focus on the nonparametric two-sample testing problem, as it represents one of the most successful applications of the mean-embedding-based MMD and its variants. The code is available at \url{https://github.com/MashaNaslidnyk/kqe}.

\subsection{Benchmarking}
We consider the following distances as test statistics in our experiments. Detail descriptions of these estimators are provided in \Cref{appendix_sec:estimators}. For KQDs, we take the reference measure $\xi$ (c.f. Proposition~\ref{res:sampling_from_gm}) to be $\nicefrac{1}{2} P_n + \nicefrac{1}{2} Q_n$, where $P_n$ corresponds to the empirical distribution $\nicefrac{1}{n}\sum_{i=1}^n \delta_{x_i}$, analogously for $Q_n$. Such $\xi$ is a general choice that is appropriate in the absence of additional information about the space $\calX$. We take power $p=2$ for all KQD-based discrepancies in our experiments; identical experiments for $p=1$ lead to the same conclusions and are presented for completeness in~\Cref{sec:exp_for_p1}. Other than in the second experiment, we use the RBF kernel $k(x,x') = \exp(-\|x-x'\|^2/2\sigma^2)$ with $\sigma$ the bandwidth chosen using the median heuristic method, i.e. $\sigma = \operatorname{Median}(\{\|x_i - x_j\|_2^2, ~\forall i,j\in 1,\dots, n\})$~\citep{gretton2012kernel}. Due to space constraints, we present all methods on the same plot, regardless of their computational complexity. However, it is important to note that directly comparing test power across methods with varying sampling complexities may be unfair and misleading.
\begin{itemize}[leftmargin=*]
\vspace{-2mm}
\setlength\itemsep{0em}
    \item \textbf{$\ekqd$ (ours).} For $\ekqd$, we set the number of projections to $l = \log n$ and the number of samples drawn from the Gaussian reference to $m=\log n$. Consequently, the overall computational complexity is $ \bigo(n \log^2(n))$.
    \item \textbf{$\ekqd$-centered (ours).} The centered version of $\ekqd$, as discussed in~\Cref{appendix:centered_quantiles}, can be expressed as the sum of an $\ekqd$ term and the classical MMD. While the $\ekqd$ component follows the same sampling configuration as above, the MMD computation is the dominant factor in complexity, leading to an overall cost of $\bigo(n^2)$.
    \item \textbf{$\supkqd$ (ours).} $\supkqd$ adopts the same sampling configuration as $\ekqd$ (thus cost $\bigo(n\log^2(n))$). Instead of averaging over projections, it selects the maximum across all projections. This approach serves as a fast approximation of the kernel max-sliced Wasserstein distance of \citet{wang_two-sample_2022}, where a Riemannian block coordinate descent method is used to optimise an entropic regularised objective at a computational cost of $\bigo(n^3 \log(n))$. In contrast, our approach identifies the largest directional quantile difference across the sampled projections. While we do not claim that this provides an accurate estimate of the true distance, this approach allows for controlled complexity and facilitates comparisons between averaging or taking the supremum.

    \item \textbf{MMD.} The MMD is included as a benchmark to be compared with $\ekqd$-centered and has complexity $\bigo(n^2)$. The MMD is estimated using the U-statistic formulation.

    \item \textbf{MMD-Multi.} A fast MMD approximation based on incomplete U-statistic introduced in \citet{Schrab2022} is included to benchmark against our $\ekqd$ distance. Configurations of MMD-Multi are chosen as to match the complexity of $\ekqd$ for a fair comparison.
    \item \textbf{MMD-Lin.} MMD-Linear from \citet[Lemma 14.]{gretton2012kernel} estimates the MMD with complexity $\bigo(n)$.
\end{itemize}

\subsection{Experimental Setup and Results}
\label{sec:experimental_results}

We conduct four experiments: two using synthetic data, allowing full control over the simulation environment, and two based on high-dimensional image data to showcase the practicality and competitiveness of our proposed methods. Additional experiments are reported in~\Cref{appendix:experiments}, specifically: studying the impact of changing the measures $\nu$ and $\xi$, comparing with sliced Wasserstein distances, and comparing with MMD based on other KME approximations.

\textbf{1. Power-decay experiment.} This experiment investigates the effect of the curse of dimensionality on our tests, following the setup of Experiment A in \citet{wang_two-sample_2022}. Prior work by \citet{ramdas2015decreasing} has shown that MMD-based methods are particularly vulnerable to the curse of dimensionality. Here, we assess whether our quantile-based test statistic exhibits similar limitations.

We fix $n=200$ and take $P$ to be an isotropic Gaussian distribution of dimension $d$. Similarly, we take $Q$ to be a $d$-dimensional Gaussian distribution with a diagonal covariance matrix $\Sigma = \operatorname{diag}(\{4, 4, 4, 1, \dots, 1\})$. As we increase the dimension $d \in [32, 64, 128, 256, 512]$, the testing problem becomes increasingly challenging. Figure~\ref{fig: main_experiments}a presents the results. We observe that $\ekqd$ exhibits the slowest decline in test power among all methods, irrespective of their computational complexity. Notably, it maintains its performance significantly better than its $\bigo(n \log^2(n))$ benchmark, MMD-Multi. These results suggest that quantile-based discrepancies exhibit greater robustness to high-dim data.

\textbf{2. Laplace v.s. Gaussian.} This experiment aims to illustrate~\Cref{res:if_meanchar_then_quantchar} by demonstrating that while a kernel may not be mean-characteristic---meaning it cannot distinguish between two distributions using standard KMEs and MMDs---it can still be quantile-characteristic. In such cases, the distributions can still be effectively distinguished using our KQEs and KQDs. To demonstrate this, we take $P$ to be a standard Gaussian in $d=1$, and $Q$ to be a Laplace distribution with matching first and second moment. We vary $n \in \{100, 500, 2000, 5000, 10000\}$ and select a polynomial kernel of degree $3$, i.e. $k(x,x') = (\langle x,x' \rangle + 1)^3$, for all our methods. This ensures that $k$ cannot distinguish between the two distributions due to their matching first and second moments, which leads to their KMEs being identical.

Figure~\ref{fig: main_experiments}b shows that our KQDs, irrespective of their computational complexity, exhibit increasing test power as the sample size grows. In contrast, MMD-based methods fail entirely to detect any differences between $P$ and $Q$. Notably, although $\ekqd$-centered can be expressed as the sum of an MMD term and an a $\ekqd$ term, the underperformance of the MMD component in this scenario is effectively compensated by the $\ekqd$ term, enabling successful testing.

\textbf{3. Galaxy MNIST.} We examine performance on real-world data through galaxy images~\citep{walmsley2022galaxy} in dimension $d=3\times64\times64=12288$, following the setting from \citet{biggs2024mmd}. These images consist of four classes. $P$ corresponds to images sampled uniformly from the first three classes, while $Q$ consists of samples from the same classes with probability $0.85$ and from the fourth class with probability $0.15$. A Gaussian RBF kernel with bandwidth chosen using the median heuristic method is chosen for all estimators. Sample sizes are chosen from $n\in\{100, 500, 1000, 1500, 2000, 2500\}$.

Figure~\ref{fig: main_experiments}c presents the results.
$\ekqd$-centered and MMD exhibit nearly identical performance, suggesting that the MMD term is dominating in the $\ekqd$-centered estimator.
Among the near-linear time test statistics, $\ekqd$ and $\supkqd$ show a slight advantage over MMD-Multi in distinguishing between the distributions of Galaxy images.

\textbf{4. CIFAR-10 v.s. CIFAR-10.1.} We conclude with an experiment on telling apart the CIFAR-10~\citep{Krizhevsky12:ImageNet} and CIFAR-10.1~\citep{recht2019imagenet} test sets, following again \citet{liu2020learning} and \citet{biggs2024mmd}. The dimension is $d = 3 \times 32 \times 32 = 3072$. This is a challenging task, as CIFAR-10.1 was designed to provide new samples from the CIFAR-10 distribution, making it an alternative test set for models trained on CIFAR-10. We conduct the test by drawing $n$ samples from CIFAR-10, and $n$ samples from CIFAR-10.1, with $n \in \{100, 500, 1000, 1500, 2000\}$.

Figure~\ref{fig: main_experiments}d presents the results. Consistent with previous observations, test statistics with quadratic computational complexity exhibit nearly identical performance. However, our quantile discrepancy estimators with near-linear complexity significantly outperform the fast MMD estimators (MMD-Multi) of the same complexity, highlighting the practical advantages of our methods in real-world testing scenarios where computational efficiency is a critical consideration.

An empirical runtime comparison of all methods is presented in Figure~\ref{fig:run_time}, which shows the time (in seconds) required to complete this experiment. The empirical results align with our complexity analysis: the near-linear estimators exhibit comparable performance, while the quadratic estimators are significantly slower. The proposed near-linear KQD estimator makes it suitable for larger-scale datasets.

\begin{figure}[!t]
    \centering
    \includegraphics[width=0.85\linewidth]{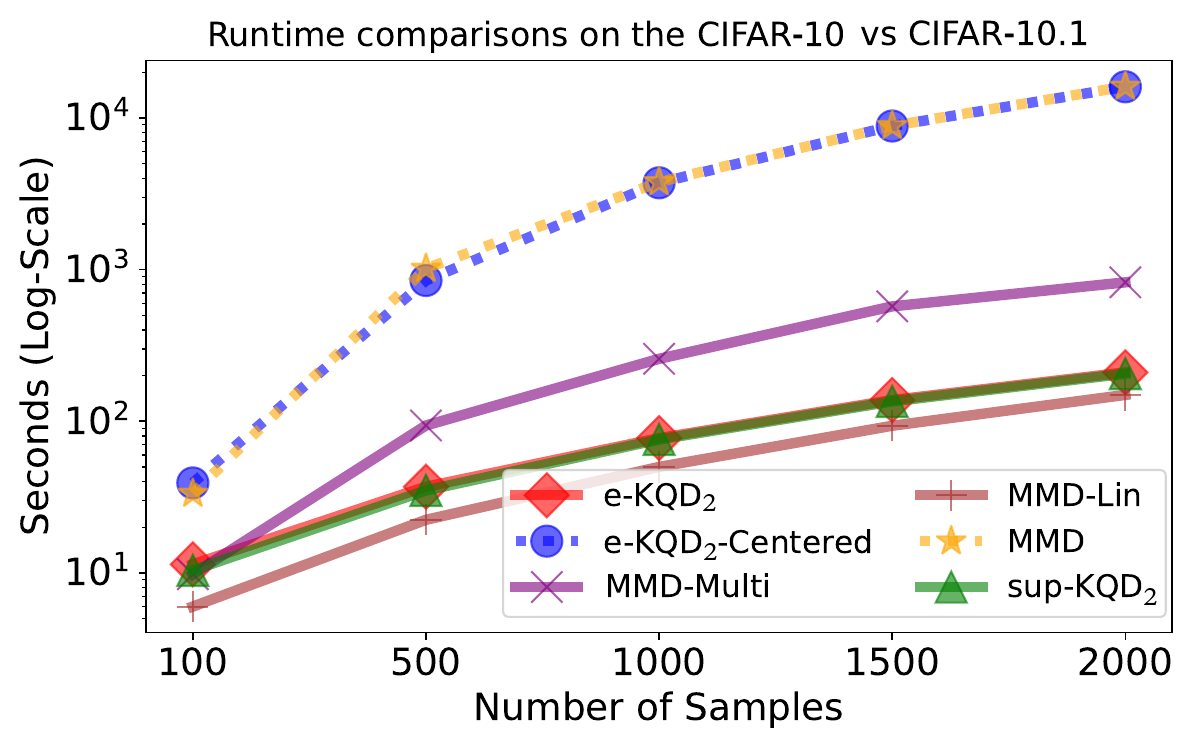}
    \vspace{-3mm}
    \caption{Comparing the time (in seconds) required to complete the CIFAR-10 vs. CIFAR-10.1 experiment, plotted on a logarithmic scale. A shorter time indicates a faster algorithm. These results align with our complexity analysis.}
    \vspace{-1em}
    \label{fig:run_time}
\end{figure}

\section{Discussion and Future Work}

This work explores representations of distributions in a RKHS beyond the mean, using functional quantiles to capture richer distributional characteristics. We introduce kernel quantile embeddings (KQEs) and their associated kernel quantile discrepancies (KQDs), and establish that the conditions required for KQD to define a distance are strictly more general than those needed for MMD to be a distance. Additionally, we propose an efficient estimator for the expected KQD based on Gaussian measures, and demonstrate its effectiveness compared to MMD and its fast approximations through extensive experiments in two-sample testing.
Our findings demonstrates the potential of KQEs as a powerful alternative to traditional mean-based representations.

Several promising avenues remain. Firstly, future work could explore more sophisticated methods for improving the empirical estimates of KQEs.
The study of optimal kernel selection to maximize test power when using KQD for hypothesis testing, analogous to existing work on MMDs~\citep{jitkrittum2020testing,liu2020learning,schrab2023mmd} could also be explored.
Secondly, considering the demonstrated potential of functional quantiles for representing marginal distributions, it is natural to ask whether they could provide a powerful alternative to conditional mean embeddings (CMEs) \citep{Song10:KCOND,Park20:CME}, the Hilbert space representation of conditional distributions. These complementary developments will unlock new avenues for enhancing existing applications of KMEs across a wide range of domains, including not only nonparametric two-sample testing, but also (conditional) independence testing, causal inference, reinforcement learning, learning on distributions, generative modeling, robust parameter estimation, and Bayesian representations of distributions via kernel mean embeddings, as explored in \citet{flaxman2016bayesian,chau2021deconditional,chau2021bayesimp}, among others.

\section*{Acknowledgements}

The authors are grateful to Carlo Ciliberto and Antonin Schrab for fruitful discussions on Gaussian measures and MMD two-sample testing respectively.
MN acknowledges support from the U.K. Research and Innovation under grant number EP/S021566/1, and from the Helmholtz Information \& Data Science Academy (HIDA) for providing financial support enabling a short-term research stay at CISPA (Application No. 14773).

\section*{Impact Statement}

This paper presents work whose goal is to advance the field of Machine Learning. There are many potential societal consequences of our work, none which we feel must be specifically highlighted here.

\bibliography{references.bib}
\clearpage

\onecolumn

\appendix

{\hrule height 1mm}
\vspace*{-0pt}
\section*{\LARGE\bf \centering Supplementary Material
}
\vspace{8pt}
{\hrule height 0.1mm}
\vspace{10pt}

This supplementary material is structured as follows. In \Cref{appendix_sec:estimators}, we recall existing probability metrics, define alternative KQDs, and then describe their respective finite-sample estimators. In \Cref{appendix:proofs}, we provide the proofs of all theoretical results in the paper. In \Cref{appendix:experiments}, we provide additional numerical experiments to complement the main text.

\section{Probability Metrics and Their Estimators }\label{appendix_sec:estimators}

\subsection{Maximum Mean Discrepancy}

We first recall that the MMD is an integral probability metric \citep{muller1997integral} where the supremum can be obtained in closed-form:
\begin{align*}
    \text{MMD}(P,Q) & \coloneqq \sup_{\|f\|_{\calH} \leq 1} \left| \mathbb{E}_{X \sim P}[f(X)] - \mathbb{E}_{Y \sim Q}[f(X)] \right| = \|\mu_P - \mu_Q \|_{\mathcal{H}}
\end{align*}
Using the reproducing property, we can then expressed the squared-MMD as
\begin{align}
\label{eq:mmd}
    \text{MMD}^2(P,Q) & \coloneqq \bE_{X,X' \sim P}[k(X,X')]
    - 2 \bE_{X \sim P,Y \sim Q}[k(X,Y)] + \bE_{Y,Y' \sim Q}[k(Y,Y')]
\end{align}
This section describes the most widely used estimators for this quantity based on i.i.d. samples $x_{1:n} \sim P$ and $y_{1:n}\sim Q$. The first is a biased V-statistic estimator with computational complexity $\bigo(n^2)$ and convergence rate $\bigo(n^{-\nicefrac{1}{2}})$:
\begin{align*}
    \text{MMD}_V^2(P,Q) \coloneqq \frac{1}{n^2}\sum_{i=1}^n \sum_{j=1}^n k(x_i,x_j) - \frac{2}{n^2} \sum_{i=1}^n \sum_{j=1}^n k(x_i,y_j) + \frac{1}{n^2}\sum_{i=1}^n \sum_{j=1}^n k(y_i,y_j)
\end{align*}
Alternatively, an unbiased estimator can be constructed through a U-statistic. Such an estimator also has computational complexity $\bigo(n^{2})$ and convergence rate $\bigo(n^{-\nicefrac{1}{2}})$~\citep[Lemma 6;Corollary 16]{gretton2012kernel}, and is given by
\begin{align*}
    \text{MMD}_U^2(P,Q) \coloneqq \frac{1}{n(n-1)}\sum_{i\neq j} k(x_i,x_j) - \frac{2}{n^2} \sum_{i=1}^n \sum_{j=1}^n k(x_i,y_j) + \frac{1}{n(n-1)}\sum_{i\neq j} k(y_i,y_j)
\end{align*}
A cheaper estimator was proposed in Lemma 14 of \cite{gretton2012kernel}. This estimator has computational complexity $\bigo(n)$ and convergence rate $\bigo(n^{-\nicefrac{1}{2}})$, and we will refer to it as \emph{MMD-Lin}. The estimator is given by:
\begin{align*}
    \text{MMD}_\text{lin}^2(P,Q) \coloneqq \frac{1}{\lfloor n/2\rfloor} \sum_{i=1}^{n/2} k(x_{2i-1},x_{2i})+k(y_{2i-1},y_{2i})-k(x_{2i-1},y_{2i})-k(x_{2i},y_{2i-1})
\end{align*}
Finally, we also studied estimators whose computational complexity are between that of MMD-lin and the U- or V-statistics estimators. These estimators are due to \citet{Schrab2022} and we refer to them as \textit{MMD-Multi}, and takes the following form
\begin{align*}
    \text{MMD}_\text{Multi}^2(P,Q) \coloneqq \frac{2}{r(2n-r-1)}\sum_{j=1}^r\sum_{i=1}^{n-j} k(x_i, x_{i+j}) + k(y_i, y_{i+j}) - k(x_i, y_{i+j}) - k(x_{i+j}, y_i)
\end{align*}
where $r$ is the number of subdiagonal considered. In our experiments, to match the complexity with $\ekqd$, we set $r = \log^2(n)$. They have computational complexity $\bigo(rn)$.

Note that several estimators with faster convergence rates exist \citep{Niu2021,Bharti2023}, but these have computational cost ranging from $\bigo(n^2)$ to $\bigo(n^{3})$ and require more regularity conditions on $k, P$ and $Q$, and we therefore omit them from our benchmark. \citet{Bodenham2023} also introduced an estimator with computational complexity of $\bigo(n \log(n))$ (and convergence $\bigo(n^{-\nicefrac{1}{2}})$) using slices/projections to $d=1$. However, their approach is restrictive in that it can only be used for the Laplace kernel, and we therefore also do not compare to it.

\subsection{Wasserstein Distance}

The $p$-Wasserstein distance \cite{Kantorovich1942,villani2009optimal} is defined as
\begin{align*}
    W_p(P,Q) \coloneqq \left(\inf_{\pi \in \Pi(P,Q)} \bE_{(X,Y)\sim \pi}\left[\rho(X,Y)^p \right]\right)^{\nicefrac{1}{p}}.
\end{align*}
Given samples $x_{1:n} \sim P$ and $y_{1:n} \sim Q$, this distance can be approximated using a plug-in $W_p(\nicefrac{1}{n}\sum_{i=1}^n \delta_{x_i},\nicefrac{1}{n} \sum_{i=1}^n \delta_{y_i})$, which can be computed in closed-form at a cost of $\bigo(n^3)$, but converges to $W_p(P,Q)$ with a convergence rate $\bigo(n^{-\nicefrac{1}{d}})$.

When $\calX \subseteq \mathbb{R}^d$ and $p=1$, we obtain the $1$-Wasserstein distance which, similarly to the MMD, can be written as an integral probability metric \citep{muller1997integral}:
\begin{align*}
    W_1(P,Q) \coloneqq \sup_{\|f\|_\text{Lip} \leq 1} \left| \mathbb{E}_{X\sim P}[f(X)] - \mathbb{E}_{X \sim Q} [f(X)]\right|
\end{align*}
where $\|f\|_{\text{Lip}} = \sup_{x,y \in \calX, x \neq y} |f(x)-f(y)|/\|x-y\|$ denotes the Lipschitz norm.

When $P,Q$ are distributions on a one dimensional space $\calX \subseteq \mathbb{R}$ that have $p$-finite moments, the $p$-Wasserstein distance can be expressed in terms of distance between quantiles of $P$ and $Q$ (see for instance~\citet[Remark 2.30]{peyre2019computational})
\begin{align}
\label{eq:wasserstein_as_quantiles}
    W_p(P,Q) = \left( \int_0^1 |\rho^{\alpha}_P - \rho^{\alpha}_Q|^p \d \alpha \right)^{\nicefrac{1}{p}}
\end{align}
A natural estimator for the Wasserstein distance is therefore based on approximating these one-dimensional quantiles using order statistics. Given $x_{1:n} \sim P$ and $y_{1:n} \sim Q$, denote by $P_n = \nicefrac{1}{n}\sum_{i=1}^n \delta_{x_i}$ and $Q_n = \nicefrac{1}{n} \sum_{i=1}^n \delta_{y_i}$ the corresponding empirical approximations to $P$ and $Q$. Then, the $\left\lceil \nicefrac{j}{n} \right\rceil$-th quantiles of $P_n$ and $Q_n$ are exactly the $j$-th order statistics $[x_{1:n}]_j$ and $[y_{1:n}]_j$, meaning the $j$'th smallest elements of $x_{1:n}$ and $y_{1:n}$ respectively. Then $W_p(P_n, Q_n)$ takes the exact form
\begin{align}
\label{eq:wasserstein_as_order_statistics}
    W_p(P_n, Q_n) = \left(\sum_{j=1}^n \left| [x_{1:n}]_j - [y_{1:n}]_j\right|^p \right)^{\nicefrac{1}{p}},
\end{align}
and is an estimator of $W_p(P, Q)$. This estimator costs $\bigo(n \log(n))$ to compute (due to the cost of sorting $n$ data points), and a convergence rate of $\bigo(n^{-\nicefrac{1}{2}})$ for $p=1$, and minimax convergence rate $\bigo(n^{-\nicefrac{1}{2p}})$ for integer $p>1$ when $P, Q$ have at least $2p$ finite moments. In some cases, the $p>1$ rate can be improved upon to match the $\bigo(n^{-\nicefrac{1}{2}})$ rate of $p=1$: we refer to~\citet{bobkov2019one} for a thorough overview.

\subsection{Sliced Wasserstein}

The \emph{sliced Wasserstein} (SW) distances \citep{rabin2011wasserstein,bonneel2015sliced} between two distributions $P,Q$ on $\mathbb{R}^d$ use one-dimensional projections to reduce computational cost. \paragraph{Expected SW.} For an integer $p \ge 1$, expected SW is defined as
\begin{align*}
    \text{SW}_p(P,Q)
    &\coloneqq \left(\bE_{u \sim \bU(S^{d-1})} [W_p^p(\phi_u\# P, \phi_u\# Q)] \right)^{\nicefrac{1}{p}},
\end{align*}
where $\bU(S^{d-1})$ is the uniform distribution on the unit sphere $S^{d-1}$, the measures $\phi_u\# P$, $\phi_u\# Q$ are pushforwards under the projection operator $\phi_u(x) = \langle u, x \rangle$, and $W_p$ is the one-dimensional $p$-Wasserstein distance as in Equation~\eqref{eq:wasserstein_as_quantiles}. Given $x_{1:n} \sim P$ and $y_{1:n} \sim Q$, the integral over the sphere is approximated by Monte Carlo sampling of $l$ directions $u_{1:l}$, which together with the estimator in~\Cref{eq:wasserstein_as_order_statistics} gives
\begin{align*}
    \widehat{\text{SW}}_p^p(P,Q) = \frac{1}{l}\sum_{i=1}^l W_p^p(\phi_{u_j}\# P_n, \phi_{u_j}\# Q_n) = \frac{1}{ln}\sum_{i=1}^l \sum_{j=1}^n \left(\big[\langle u_i, x_{1:n}\rangle\big]_j - \big[\langle u_i, y_{1:n}\rangle\big]_j \right)^p
\end{align*}
Here, $[\langle u_i, x_{1:n}\rangle]_j$ is the $j$-th order statistics, meaning the $j$-th smallest element of $\langle u_i, x_{1:n}\rangle=[\langle u_i, x_1\rangle, \dots, \langle u_i, x_n\rangle]^\top$. This estimator can be computed in $\bigo(l n\log n)$ time (the cost on sorting $n$ samples, for $l$ directions) and was shown to converge at rate $\bigo(l^{-\nicefrac12} + n^{-\nicefrac12})$ for $p=1$ \citep{nadjahi_statistical_2022}.

\paragraph{Max SW.}
The \emph{max-sliced Wasserstein} (max-SW) distance \citep{Deshpande2018} replaces the average over projections in expected SW with a supremum over directions,
\begin{align*}
    \text{max-SW}_p(P,Q)
    &\coloneqq \left(\sup_{u \in S^{d-1}} W^p_p(\phi_u \# P, \phi_u \# Q) \right)^{\nicefrac{1}{p}},
\end{align*}
where, $\phi_u(x) = \langle u, x \rangle$ is again the projection operator, and $W_p$ is the one-dimensional $p$-Wasserstein distance of~\Cref{eq:wasserstein_as_quantiles}. Max-SW emphasizes the direction of greatest dissimilarity between the two measures.

Given $x_{1:n} \sim P$ and $y_{1:n} \sim P$, max-SW is estimated as $W_p(\phi_{u^*} \# P, \phi_{u^*} \# Q)$, for $u^*$ the projection that maximises $W^p_p(\phi_u \# P_n, \phi_u \# Q_n)$ as given in~\Cref{eq:wasserstein_as_order_statistics}. In~\citep{Deshpande2018}, $u^*$ was approximated by optimising a heuristic, rather than the actual $W^p_p(\phi_u \# P_n, \phi_u \# Q_n)$. Then,~\citet{kolouri_generalized_2022} approached the actual problem of
\begin{align*}
    u^*
    &= \argmax_{\|u\|=1} W^p_p(\phi_u \# P_n, \phi_u \# Q_n).
\end{align*}
by running projected gradient descent on $S^{d-1}$, where each gradient step requires computing the derivative of the 1D Wasserstein distance w.r.t. $u$. Concretely, they initialise $u_1$ randomly and iterate
\begin{align}
\label{eq:max_sw_optimisation}
    u_{t+1} = \mathrm{Proj}_{S^{d-1}}\Big(\mathrm{Optim}\big( \nabla_u W^p_p(\phi_{u_t} \# P, \phi_{u_t} \# Q), u_{1:t}\big)\Big),
\end{align}
where $\mathrm{Proj}_{S^{d-1}}(x) = x/\|x\|$ is the operator projecting onto the unit sphere, and $\mathrm{Optim}$ is an optimiser of choice, such as ADAM. Each evaluation of $W_p$ and its gradient in one dimension costs $\bigo(n\log n)$, so the overall complexity is $\bigo(T n\log n)$ for $T$ gradient steps. It is important to point out the optimisation may be noisy, with the value objective getting worse after some iterations. Indeed, if $z_{t+1}$ is the solution to $\mathrm{Optim}\big( \nabla_u W^p_p(\phi_{u_t} \# P, \phi_{u_t} \# Q), u_{1:t}\big)$, is it an improvement over $u_t$, meaning $W^p_p(\phi_{u_t}\#P, \phi_{u_t}\#Q) \leq W^p_p(\phi_{z_{t+1}}\#P, \phi_{z_{t+1}}\#Q)$. Written out explicitly,
\begin{align*}
    \sum_{j=1}^n \left| [\langle u_t, x_{1:n} \rangle]_j - [\langle u_t, y_{1:n} \rangle]_j\right|^p \leq \sum_{j=1}^n \left| [\langle z_{t+1}, x_{1:n} \rangle]_j - [\langle z_{t+1}, y_{1:n} \rangle]_j\right|^p,
\end{align*}
Then, $u_{t+1} = \mathrm{Proj}_{S^{d-1}}(z_{t+1}) = z_{t+1} / \|z_{t+1} \|$, and it may happen that $W^p_p(\phi_{u_t}\#P, \phi_{u_t}\#Q) > W^p_p(\phi_{u_{t+1}}\#P, \phi_{u_{t+1}}\#Q)$. The desired $W^p_p(\phi_{u_t}\#P, \phi_{u_t}\#Q) \leq W^p_p(\phi_{u_{t+1}}\#P, \phi_{u_{t+1}}\#Q)$ is guaranteed when $\|z_{t+1}\|^p \leq 1$, which may not happen.

\subsection{Generalised Sliced Wasserstein}

The \emph{generalised (max-)sliced Wasserstein} (GSW and max-GSW) distances \citep{kolouri_generalized_2022} extend SW and max-SW by using a family of nonlinear feature maps $\{f_\theta : \mathbb{R}^d \to \mathbb{R}\}_{\theta\in\Theta}$ instead of linear projections. Formally,
\begin{align*}
    \text{GSW}_p(P,Q)
    &\coloneqq \left( \bE_{\theta \sim \mu} W_p^p(f_\theta \# P, f_\theta \# Q ) \right)^{\nicefrac{1}{p}},\qquad \text{max-GSW}_p(P,Q)
    \coloneqq\left( \sup_{\theta \in \Theta} W_p^p(f_\theta \# P, f_\theta \# Q ) \right)^{\nicefrac{1}{p}},
\end{align*}
where $f_\theta \# P$ denotes the pushforward of $P$ by $f_\theta$ and $\mu$ is a probability measure over the parameter space $\Theta$. For $f_\theta(x)=\langle \theta,x\rangle$ and $\Theta=S^{d-1}$ with uniform $\mu$, GSW reduce to the standard SW distances. For expected GSW, sampling $\{\theta_i\}_{i=1}^l\sim\mu$ yields an estimator with the same $\bigo(l n\log n)$ computational complexity as expected SW~\citep{kolouri_generalized_2022}. For max-GSW, the projected gradient descent approach of~\Cref{eq:max_sw_optimisation} applies, at the same complexity of $\bigo(Tn \log n)$ as for max-SW.

Statistical and topological properties of GSW depend completely on the choice of the family $\{f_\theta:\theta \in\Theta\}$.~\citet{kolouri_generalized_2022} consider the specific case of polynomial $f_\theta$, and show GSW is then a metric on probability distributions on $\Re^d$.

\subsection{Kernel Sliced Wasserstein}

A special case of the GSW arises when the feature maps $f_\theta$ are drawn from a reproducing kernel Hilbert space (RKHS). Let $k:\mathbb{R}^d\times\mathbb{R}^d\to\mathbb{R}$ be a positive definite kernel that induces the RKHS $\calH$ with unit sphere $S_\calH$. Then, the \emph{kernel sliced Wasserstein} (KSW) can be introduced as
\begin{align*}
    \text{e-KSW}_p(P,Q)
    &\coloneqq \left( \bE_{u \sim \gamma} W_p^p(u \# P, u \# Q ) \right)^{\nicefrac{1}{p}},\qquad \text{max-KSW}_p(P,Q)
    \coloneqq\left( \sup_{u \in S_\calH} W_p^p(u \# P, u \# Q ) \right)^{\nicefrac{1}{p}},
\end{align*}
where $\gamma$ is some probability measure on $S_\calH$. The expected KSW is a new construct, while max-KSW was introduced in~\citet{wang_two-sample_2022}, and studied further in~\citet{wang2024statistical}; in both papers $k$ was assumed to be universal. Finding the optimal $u^*$ for max-KSW was shown to be NP-hard in~\citet{wang2024statistical}; they propose an estimator at cost $\bigo(T^{3/2}n^2)$. Though still more expensive than computing the V-statistic estimator of MMD, this is an improvement over $\bigo(Tn^3)$ in the original work of~\citet{wang_two-sample_2022}.

As pointed out in the main text, the choice of a uniform $\gamma$ in e-KSW, while seemingly natural, may not be feasible as there is no uniform or Lebesgue measure in infinite dimensional spaces. In the main paper, we propose a practical choice of $\gamma$ that facilitates an efficient estimator, and study computational cost. Further, we establish statistical and topological properties that apply to both expected and max-KSW---and do not assume a universal kernel.

\subsection{Sinkhorn Divergence}

The entropic regularisation of optimal transport leads to the \emph{Sinkhorn divergence}~\citep{cuturi2013sinkhorn, Genevay2019}. For distributions $P,Q$ and regularisation parameter $\varepsilon>0$, the entropic OT cost is defined as
\begin{align*}
    W_{p,\varepsilon}(P,Q)
    &\coloneqq \left(\inf_{\pi\in\Pi(P,Q)} \bE_{(X,Y)\sim\pi}[\|X-Y\|^p]
    +\varepsilon \mathrm{KL}(\pi \| P\otimes Q )\right)^{\nicefrac{1}{p}}.
\end{align*}
The Sinkhorn divergence then corrects for the entropic bias:
\begin{align}
\label{eq:sinkhorn}
    \mathrm{S}_{p,\varepsilon}(P,Q)
    &\coloneqq W_{p,\varepsilon}(P,Q)
    - W_{p,\varepsilon}(P,P)/2
    - W_{p,\varepsilon}(Q,Q)/2.
\end{align}
This quantity interpolates between MMD-like behavior for large $\varepsilon$ and true Wasserstein for $\varepsilon\to0$, and can be computed efficiently via Sinkhorn iterations at cost $\bigo(n^2)$ per iteration~\citep{cuturi2013sinkhorn}.

\subsection{Kernel covariance embeddings}

Kernel covariance (operator) embeddings (KCE,~\citet{makigusa2024two}) represent the distribution $P$ as the second-order moment of the function $k(X, \cdot)$, for $X \sim P$, as an alternative to the first-order moment (the kernel mean embedding). Due to being moments of the same distribution, the two share key positives and drawbacks: KCE for kernel $k$ exists if and only if KME for $k^2$ exists, and the kernel $k$ is covariance characteristic if and only if $k^2$ is mean-characteristic~\citep{bach2022information}. The divergence proposed in~\citet{makigusa2024two} is the distance between the KCE, and is estimated at $\bigo(n^3)$ due to the need to compute full eigendecomposition of the KCE in order to compute the norm. In contrast, our proposed kernel quantile embeddings (KQE) embed quantiles, and therefore the relation to the KCE comes down to matching quantiles (which always exist, and come with an efficient estimator), compared to matching the second moment in the infinite-dimensional RKHS (which may not exist, and requires eigenvalue decomposition).

\subsection{Kernel median embeddings}

The median embedding~\citep{nienkotter2022kernel} of $P$ is the geometric median of $k(X, \cdot)$, $X \sim P$ in the RKHS, meaning the RKHS element which, on average, is $L^1$-closest to the point $k(X, \cdot)$. Explicitly put, it is the function $\mathrm{med}_P \in \calH$ defined through
\begin{equation*}
    \mathrm{med}_P = \argmin_{f \in \mathcal{H}} \int_\calH \| f(\cdot) - k(x, \cdot) \|_\calH P(\d x).
\end{equation*}
The median exists for any separable Hilbert space~\citep{minsker_geometric_2015}. However, even for an empirical $P_n = \nicefrac{1}{n}\sum_{i=1}^n \delta_{x_i}$, there is no closed-form solution to this $L^1$-problem, and the median is typically approximated using iterative algorithms like Weiszfeld’s algorithm. The estimator proposed in~\citet{nienkotter2022kernel} has a computational complexity of $\bigo(n^2)$. The property of being median-characteristic, as far as the authors are aware, has not been explored, and no theoretical guarantees are available.

The connection to 1D-projected quantiles as done in KQE, even specifically the 1D-projected median, is also unclear. Expanding the understanding of geometric median embeddings is an area for future research.

\subsection{Other Related Work}

Kernel methods have also been studied in the context of quantile estimation and regression~\citep{Sheather1990,Li2007}. These methods, however, focus on using either kernel density estimation or kernel ridge regression to estimate univariate quantiles. In contrast, our focus lies in exploring directional quantiles in the RKHS, and using them to estimate distances between distributions. We introduce this idea in the following section.

\section{Connection between Centered and Uncentered Quantiles}
\label{appendix:centered_quantiles}

\begin{proposition}[Centered $\ekqd_2$]
    The $\ekqd_2$ and $\supkqd_2$ correspondence derived based on centered directional quantiles, now expressed as $\widetilde{\ekqd_2}(P, Q; \mu, \gamma)^2$ and $\widetilde{\supkqd_2}(P, Q; \mu, \gamma)^2$ can be expressed as follows,
\begin{align*}
\widetilde{\ekqd_2}(P, Q; \mu, \gamma)^2
&= \ekqd_2(P, Q; \mu, \gamma) + \operatorname{MMD}^2(P, Q) - \mathbb{E}_{u\sim\gamma}[(\mathbb{E}_{X\sim P}[u(X)] - \mathbb{E}_{Y\sim Q}[u(Y)])^2], \\
&\geq \ekqd_2(P, Q; \mu, \gamma) + \operatorname{MMD}^2(P, Q)\\
\widetilde{\supkqd_2}(P, Q; \mu, \gamma)^2
    &= \sup_{u\in S_\calH} \left(\tau_2^2(P, Q, \mu, u) - (\mathbb{E}_{X\sim P}[u(X)] - \mathbb{E}_{Y\sim Q}[u(Y)])^2\right) + \operatorname{MMD}^2(P, Q) \\
    &\geq \supkqd_2(P, Q; \mu, \gamma) + \operatorname{MMD}^2(P, Q)
\end{align*}
\end{proposition}
\begin{proof}
    Let $P, Q \in \mathcal{P}_\mathcal{X}$ be measures on some instance space $\mathcal{X}$. Further, define $\psi: x\mapsto k(x,\cdot)$, and write $P_\psi = \psi \# P$ and $Q_\psi = \psi \# Q$. Now $P_\psi$ and $Q_\psi$ are measures on the RKHS $\calH_k$. Recall the definition of centered directional quantiles in \Cref{sec:background_quantiles},
    \begin{align*}
        {\tilde \rho_{P_\psi}}^{\alpha, u} = \left(\rho^\alpha_{\phi_u \# P_\psi} - \phi_u(\mathbb{E}_{Y\sim P_\psi}[Y])\right)u + \mathbb{E}_{Y\sim P_\psi}[Y]
    \end{align*}
    Now since we are working in the RKHS $\calH_k$, the expectation term $\mathbb{E}_{Y\sim P_\psi}[Y]$ corresponds to the kernel mean embedding $\mu_{P} := \mathbb{E}_{P}[k(X, \cdot)]$, thus we can rewrite the above expression as,
    \begin{align*}
        {\tilde \rho_{P_\psi}}^{\alpha, u} = \left(\rho_{\phi_u \# P_\psi} - \langle u, \mu_P\rangle\right)u + \mu_P
    \end{align*}
    $\tilde{\rho}^{\alpha, u}_{Q_\psi}$ can be defined analogously. Now consider integrating the difference between the two centered directional quantiles along all quantile levels, leading to
    \begin{align}
        \tilde{\tau}_2(P, Q, \mu, u) = \left(\int_0^1 \|{\tilde \rho_{P_\psi}}^{\alpha, u} - {\tilde \rho_{Q_\psi}}^{\alpha, u}\|_{\calH_k}^2 \mu(d\alpha)\right)^{\frac{1}{2}} \label{eq: centered-quad-u}
    \end{align}
    We now proceed to show $\tilde{\tau}_2^2(P, Q, \mu, u)$ ,where $\mu$ is the Lebesgue measure, can be expressed as a sum between an uncentered $\ekqd_2$ term with the MMD. Starting with expanding the RKHS norm inside the integrand,
    \begin{align}
        \|{\tilde \rho_{P_\psi}}^{\alpha, u} - {\tilde \rho_{Q_\psi}}^{\alpha, u}\|_{\calH_k}^2 &= \|\underbrace{(\rho^{\alpha}_{\phi_u\# P_\psi} - \rho^{\alpha}_{\phi_u\# Q_\psi} - \langle u, \mu_P - \mu_Q \rangle)}_{=:A \in \mathbb{R}} u + \mu_P - \mu_Q\|_{\calH_k}^2 \nonumber \\
        &= \|Au + (\mu_P-\mu_Q)\|_{\calH_k}^2 \nonumber \\
        &= 2\langle Au, \mu_P-\mu_Q\rangle + \|Au\|_{\calH_k}^2 + \|\mu_P-\mu_Q\|_{\calH_k}^2 \nonumber \\
        &= 2A \langle u, \mu_P-\mu_Q\rangle + A^2 + \operatorname{MMD}^2(P, Q) \label{eq: 2-norm-diff-centered-quantiles}
    \end{align}
    Plugging the expression from \Cref{eq: 2-norm-diff-centered-quantiles} into \Cref{eq: centered-quad-u}, we get the following,
    \begin{align}
        \tilde{\tau}_2^2(P, Q, \mu, u) &= \int_0^1 (2A\langle u, \mu_P - \mu_Q\rangle + A^2)\mu(d\alpha) + \operatorname{MMD}^2(P, Q) \nonumber \\
        &= 2\langle u, \mu_P - \mu_Q\rangle\int_0^1 A \mu(d\alpha) + \int_0^1 A^2 \mu(d\alpha) + \operatorname{MMD}^2(P, Q) \label{eq: centered-tau}
    \end{align}
    For the first term on the right hand side, notice that,
    \begin{align}
        \int_0^1 A \mu(d\alpha) = \int_0^1 (\rho^{\alpha}_{\phi_u\# P_\psi} - \rho^{\alpha}_{\phi_u\# Q_\psi} - \langle u, \mu_P - \mu_Q \rangle) \mu(d\alpha) \label{eq: integrating-quantile}
    \end{align}
    Recall standard results from probability theory that integrating the quantile function between $0$ to $1$ with the Lebesgue measure returns you the expectation, specifically, that is,
    \begin{align*}
        \int_0^1 \rho^\alpha_{\phi_u\# P_\psi} \mu(d\alpha) = \mathbb{E}_{X \sim P}[u(X)] = \langle u, \mu_P\rangle.
    \end{align*}
    Using this fact, the terms in \Cref{eq: integrating-quantile} cancels out, leaving $\int_0^1 A\mu(d\alpha) = 0$. Therefore, continuing from \Cref{eq: centered-tau}, we have,
    \begin{align*}
        \tilde{\tau}_2^2(P, Q, \mu, u) &= \int_0^1 A^2 \mu(d\alpha) + \operatorname{MMD}^2(P, Q) \\
        &= \int_0^1 (\rho^{\alpha}_{\phi_u\# P_\psi} - \rho^{\alpha}_{\phi_u\# Q_\psi} - \langle u, \mu_P - \mu_Q \rangle)^2 \mu(d\alpha) + \operatorname{MMD}^2(P,Q) \\
        &= \int_0^1 \|(\rho_{s_{\mu_P,u}\# (\phi_u \# P_\psi)}^{\alpha} - \rho_{s_{\mu_Q,u}\# (\phi_u \# Q_\psi)}^{\alpha})u\|^2 \mu(d\alpha) + \operatorname{MMD^2}(P,Q)
    \end{align*}
    where $s_{\mu_P,u}: \mathbb{R} \to \mathbb{R}$ is a shifting function defined as $s_{\mu_P,u}(r) = r - \langle u, \mu_p \rangle$ for $r\in\mathbb{R}$. Alternatively, after expanding the terms in $A^2$, we can express $\tilde{\tau}_2^2(P, Q, \mu, u)$ as,
    \begin{align*}
        \tilde{\tau}_2^2(P, Q, \mu, u) &= \int_0^1 (\rho_{\phi_u\# P_\psi} - \rho_{\phi_u\#Q_\psi})^2 \mu(d\alpha) - (\mathbb{E}[u(X) - u(Y)])^2 + \operatorname{MMD}^2(P,Q) \\
        &= \tau_2^2(P, Q, \mu, u) + \operatorname{MMD}^2(P, Q) - (\mathbb{E}[u(X) - u(Y)])^2
    \end{align*}
As a result, for $\gamma$ a measure on the unit sphere of $\calH_k$, the centered version of $\ekqd_2$ and $\supkqd_2$, now expressed as $\widetilde{\ekqd_2}$ and $\widetilde{\supkqd_2}$, are given by,
\begin{align*}
    \widetilde{\ekqd_2}(P, Q; \mu, \gamma)^2 &= \mathbb{E}_{u\sim \gamma}\left[\tilde{\tau}_2^2(P, Q; \mu, u)\right] \\
    &= \ekqd_2(P, Q; \mu, \gamma)^2 + \operatorname{MMD}^2(P, Q) - \mathbb{E}_{u\sim\gamma}[(\mathbb{E}_{X\sim P}[u(X)] - \mathbb{E}_{Y\sim Q}[u(Y)])^2], \\
    &\leq \ekqd_2(P, Q; \mu, \gamma)^2 + \operatorname{MMD}^2(P, Q) \\
    \widetilde{\supkqd_2}(P, Q; \mu, \gamma)^2 &= \sup_{u\in S_\calH} \tilde{\tau}_2^2(P, Q; \mu, u) \\
    &= \sup_{u\in S_\calH} \left(\tau_2^2(P, Q, \mu, u) - (\mathbb{E}[u(X)] - \mathbb{E}[u(Y)])^2\right) + \operatorname{MMD}^2(P, Q) \\
    &\leq \sup_{u\in S_\calH}\tau_2^2(P, Q; \mu, u) - \sup_{u\in S_\calH}(\mathbb{E}[u(X)] - \mathbb{E}[u(Y)])^2 + \operatorname{MMD}^2(P, Q) \\
    &\leq \supkqd_2(P, Q; \mu, \gamma)^2 + \operatorname{MMD}^2(P, Q).
\end{align*}

\end{proof}

When $\nu\equiv \mu$ and the connections to Sliced Wasserstein explored in~\Cref{res:connections_slicedwasserstein} and~\Cref{res:connections_maxslicedwasserstein} emerges, the mean-shifting property of Wasserstein distances allows us to express centered KQD as a sum of uncentered KQD, and MMD---a curious interpretation of centering.

\section{Proof of Theoretical Results}\label{appendix:proofs}

This section now provides the proof of all theoretical results in the main text.
\subsection{Proof of~\Cref{res:cramer-wold}}
\label{sec:proof_projections_determine_distribution}

The main result in this section,~\Cref{res:projections_determine_distribution}, shows that the set of $\Re$ measures $\{u \#P : u \in S_\calH \}$ fully determines the distribution $P$. Since quantiles determine the distribution,~\Cref{res:cramer-wold} follows immediately.

Being concerned with the RKHS case specifically allows us to prove the result under mild conditions by using~\emph{characteristic functionals}, an extension of characteristic functions to measures on spaces beyond $\mathbb{R}^d$. Characteristic functionals describe Borel probability measures as operators acting on some function space $\calF: \calX \to \Re$.
\begin{definition}[\citet{vakhania1987probability}, Section IV.2.1]
    The~\emph{characteristic functional} $\varphi_P: \calF \to \mathbb C$ of a Borel probability measure $P$ on $\calX$ is defined as
    \begin{equation*}
        \varphi_P(f) = \int_\calX e^{if(x)} P (\d x).
    \end{equation*}
\end{definition}

Theorem 2.2(a) in~\citet[Chapter 4]{vakhania1987probability} establishes that a $P$-characteristic functional on $\calF$ uniquely determines the distribution $P$---on the smallest $\sigma$-algebra under which all function $f \in \calF$ are measurable. Therefore, when $\calF$ is such that this $\sigma$-algebra coincides with the Borel $\sigma$-algebra, the distribution is fully determined by $P$-characteristic functional on $\calF$. We show that, indeed, this holds in our setting, for $\calF=\calH$.
\begin{lemma}
    Suppose~\Cref{as:input_space,as:kernel} holds. Then, the Borel $\sigma$-algebra $\calB(\calX)$ is the smallest $\sigma$-algebra on $\calX$ under which all functions $f \in \calH$ are measurable.
\end{lemma}
\begin{proof}
    Denote by $\hat C(\calX, \calH)$ the smallest $\sigma$-algebra on $\calX$ under which all functions $f \in \calH$ are measurable, and recall that the Borel $\sigma$-algebra is the $\sigma$-algebra that contains all closed sets. Therefore, we need to show that $\hat C(\calX, \calH)$ contains every closed set in $\calX$. We split the proof into two parts: (1) show that $\calH$ contains a countable separating subspace, and (2) show that this implies that every closed set lies in $\hat C (\calX, \calH)$.
    \paragraph{$\calH$ contains a countable separating subspace.} Recall that a function space $\calF$ on $\calX$ is said to be separating when for any $x_1 \neq x_2 \in \calX$, there is a function $f \in \calF$ such that $f(x_1) \neq f(x_2)$. Since $k$ is separating, $\calH$ is separating. Since $\calH$ is separable, it contains a countable dense subspace $\calH_0 \subseteq \calH$. By $\calH_0$ being dense in $\calH$, it must also be separating.
    \paragraph{Every closed set lies in $\hat C (\calX, \calH)$.} By~\citet[Section I.1, Exercise 9]{vakhania1987probability}, all compact sets in $\calX$ lie in $\hat C(\calX, \calH_0)$, by $\calH_0$ being countable, continuous, separating space of real-valued functions. By definition, $\hat C(\calX, \calH_0) \subseteq \hat C(\calX, \calH)$, and so $\hat C(\calX, \calH)$ contains all compact sets. We now show this means every closed set must also lie in $\hat C(\calX, \calH)$.
    
    By $\calX$ being $\sigma$-compact, there is a family of compact sets $\{\calX_i\}_{i=1}^\infty$ such that $\calX = \cup_{i=1}^\infty \calX_i$. Take any closed $K \subseteq X$; then, $K = \cup_{i=1}^\infty (\calX_i \cap K)$. Since $\calX_i \cap K$ is compact as the intersection of a compact set and a closed set, and $\sigma$-algebras are closed under countable unions, $K$ must lie in $\hat C(\calX, \calH)$. As this holds for every closed $K$, we conclude $\calB(\calX)=\hat C(\calX, \calH)$.
\end{proof}

We now restate the RKHS-specific version of the Vakhania result here for completeness.

\begin{theorem}[Theorem 2.2(a) in \citet{vakhania1987probability} for RKHS]
\label{res:char_fnal_is_char}
    Suppose~\Cref{as:input_space,as:kernel} holds, and for Borel probability measures $P, Q$ on $\calX$, it holds that $\varphi_P(f) = \varphi_Q(f)$ for every $f \in \calH$. Then, $P=Q$.
\end{theorem} 

We are now ready to prove the distribution of projections uniquely determines the distribution.

\begin{proposition}
\label{res:projections_determine_distribution}
Under~\Cref{as:input_space,as:kernel}, it holds that
\begin{equation*}
    u\#P = u\#Q \ \text{ for all } u \in S_\calH \iff P = Q.
\end{equation*}
\end{proposition}
\begin{proof}
    The main idea of the proof is to show that equality of $u\#P$ and $u\#Q$ implies equality of characteristic functionals, $\varphi_P(f)=\varphi_Q(f)$ for all $f \in \calH$ such that $f(x)=tu(x)$ for some $t \in \Re$ and $u$ in the unit sphere. Since such $f$ form the entire $\calH$, the result immediately follows.

    First, recall that $u\#P = u\#Q$ for all $u$ if and only if their characteristic functions coincide, meaning
    \begin{equation}
    \label{eq:equality_of_char_functions}
        \int_\Re e^{itz} u\#P (\d z) = \int_\Re e^{itz} u\#Q (\d z) \quad \forall u \in S_\calH,\forall t \in \Re.
    \end{equation}
    Notice that the measure $u\#P$ is a pushforward of $P$ under the map $x \to u(x)$. Then, for any measurable $g$ it holds that
    \begin{equation}
    \label{eq:pushforward_integration}
        \int_\calX g(u(x)) P(\d x) = \int_\Re g(z) u\#P(\d z) \quad \forall u \in S_\calH.
    \end{equation}
    Take $g(z) = e^{itz}$, for some $t \in \Re$. Then, for all $u$ it holds that $\int_\Re e^{itz} u\#P(\d z)= \int_\Re e^{itz} u\#Q(\d z)$, and consequently by~\eqref{eq:equality_of_char_functions} we have that
    \begin{equation}
    \label{eq:equality_of_functional_over_sphere}
        \int_\calX e^{itu(x)} P(\d x) = \int_\calX e^{itu(x)} Q(\d x)\quad \forall u \in S_\calH,\forall t \in \Re.
    \end{equation}
    Finally, let us pick an $f \in \calH$ and show that $\varphi_P(f) = \varphi_Q(f)$. Define $u = f/\|f\|$, and $t=\|f\|$; then,
    \begin{equation*}
        \varphi_P(f) = \int_\calX e^{if(x)} P(\d x) = \int_\calX e^{itu(x)} P(\d x),
    \end{equation*}
    and by~\eqref{eq:equality_of_functional_over_sphere}, we arrive at the equality of characteristic functionals, $\varphi_P(f) = \varphi_Q(f)$. By~\Cref{res:char_fnal_is_char} characteristic functionals uniquely determine the underlying distribution, meaning $P = Q$.
\end{proof}

For the sake of clarity, we give the proof of the original result.

\begin{proof}[Proof of~\Cref{res:cramer-wold}]
Suppose $\{\rho_P^{\alpha,u} : \alpha \in [0, 1], u \in S_\calH\} = \{\rho_Q^{\alpha,u} : \alpha \in [0, 1], u \in S_\calH\}$ for some Borel probability measures $P, Q$. For any fixed $u$, since every quantile of of $u\#P$ and $u\#Q$ coincide, the measures coincide as well, $u\#P=u\#Q$. As that holds for every $u$, by~\Cref{res:projections_determine_distribution}, $P=Q$.
\end{proof}

Lastly, we point out~\Cref{as:input_space} may be relaxed. Provided $\calX$ is a Tychonoff space---meaning, a completely regular Hausdorff space---part (b) of Theorem 2.2 in~\citet{vakhania1987probability} says the following.
\begin{theorem}[Theorem 2.2(b) in \citet{vakhania1987probability} for RKHS]
\label{res:char_fnal_is_char_b}
    Suppose $\calX$ is Tychonoff,~\Cref{as:kernel} holds, and for Radon probability measures $P, Q$ on $\calX$, it holds that $\varphi_P(f) = \varphi_Q(f)$ for every $f \in \calH$. Then, $P=Q$.
\end{theorem} 
Therefore, when~\Cref{as:input_space} is replaced with $\calX$ being Tychonoff,~\Cref{res:cramer-wold} continues to hold---but only for Radon $P,Q$, not any Borel $P, Q$. Radon probability measures can be intuitively seen as the "non-pathological" Borel measures---a restriction employed in order to drop the regularity assumptions of $\calX$ being separable and $\sigma$-compact.

\subsection{Proof of \Cref{res:if_meanchar_then_quantchar}}
\label{appendix:proof_if_meanchar_then_quantchar}
We prove that every mean-characteristic kernel is quantile-characteristic, and give an example quantile-characteristic kernel that is not mean-characteristic.

\paragraph{mean-characteristic $\Rightarrow$ quantile-characteristic.} Suppose $k$ on $\calX$ is mean-characteristic, and $P \neq Q$ are any probability measures on $\calX$. We will identify a unit-norm $u$ for which the sets of quantiles of $u \# P$ and $u \# Q$ differ.

Since $k$ is mean characteristic, $\mu_P \neq \mu_Q$, and $\MMD^2(P, Q) = \|\mu_P - \mu_Q \|^2_\calH > 0$. Recall that MMD can be expressed as
\begin{equation*}
        \MMD^2(P, Q) = \sup_{u \in \calH, \|u\|_\calH \leq 1}\left| \bE_{X \sim P} u(X) - \bE_{Y \sim Q} u(Y)\right|,
\end{equation*}
and the supremum is attained at $u^* = (\mu_P - \mu_Q) / \|\mu_P - \mu_Q \|_\calH$~\citep{gretton2012kernel}. In other words, $\bE_{X \sim P} u^*(X) \neq \bE_{Y \sim Q} u^*(Y)$---the means of $u^*\#P$ and $u^*\#Q$ don't coincide. Therefore, the measures $u^*\#P$ and $u^*\#Q$ don't coincide, or equivalently $\{\rho_{u^*\#P}^\alpha : \alpha \in [0,1]\} \neq \{\rho_{u^*\#Q}^\alpha : \alpha \in [0,1]\}$. Then, $\{\rho_P^{u, \alpha} : \alpha \in [0,1], u \in S_\calH\} \neq \{\rho_Q^\alpha : \alpha \in [0,1], u \in S_\calH\}$. And since this holds for any arbitrary $P \neq Q$, the kernel $k$ is quantile-characteristic.

\paragraph{quantile-characteristic $\not\Rightarrow$ mean-characteristic.}

To show the converse implication does not hold, we provide an example when $k$ is quantile-characteristic but not mean-characteristic. Take $\calX = \Re^d$, and let $k$ be a degree $T$ polynomial kernel, $k(x, x') = (x^\top x' + 1)^T$. Since~\Cref{as:input_space,as:kernel} hold---$\Re^d$ is Polish, and $k$ is trivially continuous and separating---by~\Cref{res:cramer-wold} the kernel $k$ is quantile-characteristic.

Now, we show $k$ is not mean-characteristic. Suppose $P$ and $Q$ are such that $\bE_{X \sim P} X^i = \bE_{Y \sim P} Y^i$ for $i \in \{1, \dots, T\}$---for example, the Gaussian and Laplace distribution with matching expectation and variance and $T=2$, as is done in~\Cref{sec:experimental_results}. Then, $\bE_{X \sim P} (X^\top x')^i = \bE_{Y \sim P} (Y^\top x')^i$ for any $x' \in \Re^d$, and since
\begin{align*}
    \mu_P(x') &\coloneqq \bE_{X\sim P} k(X, x') = \bE_{X\sim P} [(X^\top x' + 1)^T] = \bE_{X\sim P} \left[\sum_{i=0}^T \begin{pmatrix} T \\ i \end{pmatrix} (X^\top x')^i \right] = \sum_{i=0}^T \begin{pmatrix} T \\ i \end{pmatrix} \bE_{X\sim P} \left[(X^\top x')^i \right],
\end{align*}
it holds that $\mu_P=\mu_Q$. The kernel is not mean-characteristic.

\subsection{Proof of \Cref{res:consistency_KQE}}\label{appendix:proof_consistency_KQE}
By the Theorem in~\citet[Section 2.3.2]{serfling2009approximation}, for any $\varepsilon>0$ it holds that
\begin{equation*}
    P(| \rho^\alpha_{u \# P_n} - \rho^\alpha_{u \# P}| > \varepsilon) \leq 2 e^{-2n \delta_\varepsilon^2}, \qquad \text{for} \qquad
    \delta_\varepsilon \coloneqq \min\left\{\int_{\rho^\alpha_{u \# P}}^{\rho^\alpha_{u \# P} + \varepsilon} f_{u \# P}(t) \d t, \int_{\rho^\alpha_{u \# P}-\varepsilon}^{\rho^\alpha_{u \# P}} f_{u \# P}(t) \d t\right\}.
\end{equation*}
Since it was assumed $f_{u\#P}(x) \geq c_u > 0$, it holds that $\delta_\varepsilon \geq c_u \varepsilon$, and $P(| \rho^\alpha_{u \# P_n} - \rho^\alpha_{u \# P}| > \varepsilon) \leq 2 e^{-2 n c_u^2 \varepsilon^2}$, or equivalently,
\begin{equation*}
    P(| \rho^\alpha_{u \# P_n} - \rho^\alpha_{u \# P}| \leq \varepsilon)
    \geq 1 - 2 e^{-2 n c_u^2 \varepsilon^2}.
\end{equation*}
Take $\delta \coloneqq 2 e^{-2 n c_u^2 \varepsilon^2}$. Then,
\begin{equation*}
    P(| \rho^\alpha_{u \# P_n} - \rho^\alpha_{u \# P}| \leq C(\delta, u) n^{-1/2})
    \geq 1 - \delta, \qquad \text{for} \qquad C(\delta, u)=\sqrt{\frac{\log(2/\delta)}{2 c_u^2}}.
\end{equation*}
Since $\| \rho^{\alpha,u}_{P_n} - \rho^{\alpha,u}_{P}\|_\calH=| \rho^\alpha_{u \# P_n} - \rho^\alpha_{u \# P}|$, the proof is complete.

\subsection{Proof of \Cref{res:KQE_characterise_dists}}\label{appendix:proof_quantile_characteristic}
We prove $\ekqd$ and $\supkqd$, defined in~\Cref{eq:general_distances} as
\begin{equation*}
\begin{split}
    \ekqd_p(P, Q; \nu, \gamma) &= \left(\bE_{u \sim \gamma} \tau_p^p\left(P, Q; \nu, u \right) \right)^{\nicefrac{1}{p}},\\
    \supkqd_p(P, Q; \nu) &= \big(\sup_{u \in S_\calH} \tau_p^p\left(P, Q; \nu, u \right)\big)^{\nicefrac{1}{p}},
\end{split}
\end{equation*}
are probability metrics on the set of Borel probability measures on $\calX$. Symmetry and non-negativity hold trivially.
\paragraph{Triangle inequality.} By Minkowski inequality, for any $P, P', Q$,
\begin{align*}
    \int_0^1 \big|\rho_P^\alpha - \rho_{P'}^\alpha \big| ^p \nu(\d \alpha) &\leq \Bigg(\left(\int_0^1 \big|\rho_P^\alpha - \rho_Q^\alpha \big| ^p \nu(\d \alpha) \right)^{1/p} \\
    &\hspace{2cm}+ \left(\int_0^1 \big|\rho_Q^\alpha - \rho_{P'}^\alpha \big| ^p \nu(\d \alpha) \right)^{1/p} \Bigg)^p.
\end{align*}
Plugging this in and using Minkowski inequality again on the outermost integral, we get
\begin{align*}
    \ekqd_p(P, P'; \nu, \gamma)
    &= \left(\bE_{u \sim \gamma} \int_0^1 \big|\rho_P^\alpha - \rho_{P'}^\alpha \big| ^p \nu(\d \alpha) \right)^{1/p} \\
    &\leq \Bigg(\bE_{u \sim \gamma} \Bigg(\left(\int_0^1 \big|\rho_P^\alpha - \rho_Q^\alpha \big| ^p \nu(\d \alpha) \right)^{1/p} \\
    &\hspace{2cm}+ \left(\int_0^1 \big|\rho_Q^\alpha - \rho_{P'}^\alpha \big| ^p \nu(\d \alpha) \right)^{1/p} \Bigg)^p \Bigg)^{1/p} \\
    &\leq \Bigg(\bE_{u \sim \gamma} \int_0^1 \big|\rho_P^\alpha - \rho_Q^\alpha \big| ^p \nu(\d \alpha) \Bigg)^{1/p} \\
    &\hspace{2cm}+ \Bigg( \bE_{u \sim \gamma} \int_0^1 \big|\rho_Q^\alpha - \rho_{P'}^\alpha \big| ^p \nu(\d \alpha) \Bigg)^{1/p} \\
    &= \ekqd_p(P, Q; \nu, \gamma) + \ekqd_p(Q, P'; \nu, \gamma).
\end{align*}
Similarly, since $\sup_x f^p(x) = (\sup_x |f(x)|)^p$ for any $f$,
\begin{align*}
    \supkqd_p(P, P'; \nu, \gamma)
    &= \left(\sup_{u \in S_\calH} \int_0^1 \big|\rho_P^\alpha - \rho_{P'}^\alpha \big| ^p \nu(\d \alpha) \right)^{1/p} \\
    &\leq \Bigg(\sup_{u \in S_\calH} \Bigg(\left(\int_0^1 \big|\rho_P^\alpha - \rho_Q^\alpha \big| ^p \nu(\d \alpha) \right)^{1/p} \\
    &\hspace{2cm}+ \left(\int_0^1 \big|\rho_Q^\alpha - \rho_{P'}^\alpha \big| ^p \nu(\d \alpha) \right)^{1/p} \Bigg)^p \Bigg)^{1/p} \\
    &= \left( \sup_{u \in S_\calH} \int_0^1 \big|\rho_P^\alpha - \rho_Q^\alpha \big| ^p \nu(\d \alpha) \right)^{1/p} \\
    &\hspace{2cm}+ \left( \sup_{u \in S_\calH} \int_0^1 \big|\rho_Q^\alpha - \rho_{P'}^\alpha \big| ^p \nu(\d \alpha) \right)^{1/p} \\
    &= \supkqd_p(P, Q; \nu, \gamma) + \supkqd_p(Q, P'; \nu, \gamma).
\end{align*}

\paragraph{Identity of indiscernibles.} In the rest of this section, we show that
\begin{equation*}
    \ekqd_p(P, Q; \nu, \gamma) = 0 \iff P = Q; \qquad \text{and} \qquad
    \supkqd_p(P, Q; \nu, \gamma) = 0 \iff P = Q.
\end{equation*}
Necessity (meaning the $\Leftarrow$ direction) holds trivially---quantiles of identical measure are identical. To prove sufficiency, we only need to show that both discrepancies aggregate the directions in a way that preserves injectivity, meaning
\begin{equation*}
    \ekqd_p(P, Q)=0 \Rightarrow \rho_P^{\alpha,u}=\rho_Q^{\alpha,u} \text{ for all } \alpha, u; \qquad\text{and}\qquad \supkqd_p(P, Q)=0 \Rightarrow \rho_P^{\alpha,u}=\rho_Q^{\alpha,u} \text{ for all } \alpha, u.
\end{equation*}
Together with~\Cref{res:cramer-wold}, this will complete the proof of sufficiency.

First, we show that for any pair of probability measures, a $\nu$-aggregation over the quantiles is injective.

\begin{lemma}
\label{res:nu_aggregation_is_injective}
Let $\nu$ have full support, meaning $\nu(A)>0$ for any open $A \subset [0, 1]$. For any Borel probability measures $P', Q'$,
\begin{equation*}
    \int_0^1 | \rho_{P'}^\alpha - \rho_{Q'}^\alpha |^2 \nu(\d \alpha) = 0 \qquad \Rightarrow \qquad \rho_{P'}^\alpha = \rho_{Q'}^\alpha \quad \text{for all} \quad \alpha \in [0, 1].
\end{equation*}
\end{lemma}
\begin{proof}
    Suppose $\int_0^1 | \rho_{P'}^\alpha - \rho_{Q'}^\alpha |^2 \nu(\d \alpha) = 0$, but there is an $\alpha_0$ such that $\rho_{P'}^{\alpha_0} = \rho_{Q'}^{\alpha_0}$. We will show that this implies the existence of an open set (containing $\alpha_0$) over which $|\rho_{P'}^{\alpha_0} - \rho_{Q'}^{\alpha_0}|^2 > 0$---which will contradict $\nu$ having full support.
    
    Since $|\rho_{P'}^{\alpha_0} - \rho_{Q'}^{\alpha_0}|^2 > 0$ and the quantile function $\alpha \mapsto q_P^\alpha$ is left-continuous (by definition) for any probability measure $P$, there is a $\alpha_1<\alpha_0$ such that $|\rho_{P'}^{\alpha} - \rho_{Q'}^{\alpha}|^2 > 0$ for all $\alpha \in (\alpha_1, \alpha_0]$. Take some $\alpha_2 \in (\alpha_1, \alpha_0)$. Then, for all $\alpha \in (\alpha_1, \alpha_2)$, we have $|\rho_{P'}^{\alpha} - \rho_{Q'}^{\alpha}|^2 > 0$. We arrive at a contradiction. Such $\alpha_0$ cannot exist, and therefore $\rho_{P'}^\alpha = \rho_{Q'}^\alpha$ for all $\alpha \in [0, 1]$.
\end{proof}
This result applies directly to the directional differences $\tau_p$. Provided $\nu$ has full support,
\begin{align*}
    \tau_p(P, Q;\nu, u) = 0 \ \qquad \Rightarrow \qquad \rho_{P'}^\alpha = \rho_{Q'}^\alpha \quad \text{for all} \quad \alpha \in [0, 1].
\end{align*}
Since supremum aggregation simply considers $u$ that corresponds to the largest $\tau^p_p(P, Q;\nu, u)$, this concludes the proof for $\supkqd$.
Expectation aggregation over the directions $u$ needs an extra result, given below.
\begin{lemma}
Let $\gamma$ have full support on $S_\calH$, and $\nu$ have full support on $[0, 1]$. For any Borel probability measures $P, Q$ on $\calX$,
\begin{equation*}
    \bE_{u \sim \gamma} \tau^p_p(P, Q;\nu, u) = 0 \qquad \Rightarrow \qquad P=Q.
\end{equation*}
\end{lemma}
\begin{proof}
    Same as in the proof~\Cref{res:cramer-wold}, we will use the technique of characteristic functionals $\varphi_{P}, \varphi_{Q}$, to carefully prove equality almost everywhere with respect to a full support measure $\gamma$ implies full equality. Consider the function
    \begin{equation*}
        f \mapsto \varphi_{P}(f) - \varphi_{Q}(f),
    \end{equation*}
    which is continuous by continuity of characteristic functionals. Define $f_0 \equiv 0$, the zero function in $\calH$. The set
    \begin{equation*}
        \calH^{\setminus 0} \coloneqq \{f \in \calH \setminus \{f_0\}: \varphi_{P}(f) - \varphi_{Q}(f) \in \Re \setminus \{0\}\} = \{f \in \calH \setminus \{f_0\}: \varphi_{P}(f) \neq \varphi_{Q}(f)\}
    \end{equation*}
    is open, as a preimage of an open set $\Re \setminus \{0\}$, intersected with an open set $\{\calH \setminus \{f_0\}\}$. Since the projection map $f \mapsto f/\|f\|_\calH$ is open on $\calH \setminus \{f_0\}$, the projection of $\calH_{\setminus 0}$ onto $S_\calH$ is open. In other words, the set
    \begin{equation*}
        S_\calH^{\setminus 0} \coloneqq \{u \in S_\calH: \varphi_{P}(t_u u) \neq \varphi_{Q}(t_u u) \text{ for some } t_u \in \Re \}
    \end{equation*}
    is open in $S_\calH$. Then, by definition of characteristic functionals, for $u \in S_\calH^{\setminus 0}$ it holds that
    \begin{equation*}
        \varphi_{u\#P}(t_u) = \varphi_{P}(t_u u) \neq \varphi_{Q}(t_u u) = \varphi_{u\#Q}(t_u),
    \end{equation*}
    meaning the characteristic functions of $u\#P$ and $u\#Q$ are not identical, and therefore $u\#P \neq u\#Q$. Since $\nu$ has full support on $[0, 1]$, it follows that
    \begin{equation*}
        \tau^p_p(P, Q;\nu, u)=\int_0^1 | \rho_{u\#P}^\alpha - \rho_{u\#Q}^\alpha |^p \nu(\d \alpha) > 0,\qquad \text{ for all } u \in S_\calH^{\setminus 0}
    \end{equation*}
    We arrive at a contradiction: since $\gamma$ has full support on $S_\calH$ and $S_\calH^{\setminus 0} \subseteq S_\calH$ was shown to be an open set, it holds that
    \begin{equation*}
        \bE_{u \sim \gamma} \tau^p_p(P, Q;\nu, u) \geq \int_{S_\calH^{\setminus 0}} \tau^p_p(P, Q;\nu, u) \gamma(\d u) > 0.
    \end{equation*}
    Therefore, for $\bE_{u \sim \gamma} \tau^p_p(P, Q;\nu, u)$ to be zero, $S_\calH^{\setminus 0}$ must be empty---which, by construction, can only happen when $\calH^{\setminus 0}$ is empty, i.e. $\varphi_{P}(f) = \varphi_{Q}(f)$ for all $f \in \calH \setminus f_0$, where $f_0 \equiv 0$. Since $\varphi_{P}(f_0) = \varphi_{Q}(f_0)$ holds trivially for any $P, Q$, the characteristic functionals of $P$ and $Q$ are identical. By~\Cref{res:char_fnal_is_char}, $P=Q$. This concludes the proof.
\end{proof}

\subsection{Proof of \Cref{res:consistency_KQD}}\label{appendix:proof_consistency_KQD}
We start with two auxiliary lemmas that, when combined, bound $\ekqd$ approximation error due to replacing $P, Q$ with $P_n, Q_n$ in $n^{-1/2}$. This will be crucial in showing convergence of the approximate e-KQD to the true e-KQD.
\begin{lemma}
\label{res:lemma_for_consistency}
For any measure $\nu$ on $[0, 1]$ and any measure $\gamma$ on $S_\calH$, it holds that
\begin{align*}
    |\ekqd_1(P_n, Q_n; \nu, \gamma) - \ekqd_1(P, Q; \nu, \gamma)|
    \leq \ekqd_1(P_n, P; \nu, \gamma) + \ekqd_1(Q_n, Q; \nu, \gamma)).
\end{align*}
\end{lemma}
\begin{proof}
By the definition of $\ekqd_1$ and Jensen inequality for the absolute value,
\begin{align*}
    |\ekqd_1(P_n, Q_n; \nu, \gamma) - \ekqd_1(P, Q; \nu, \gamma)|
    &= \left|\bE_{u \sim \gamma} \left[ \int_0^1 \left( |\rho^\alpha_{u\#P_n} - \rho^\alpha_{u\#Q_n}| - |\rho^\alpha_{u\#P} - \rho^\alpha_{u\#Q}| \right) \d \alpha\right]\right| \\
    &\leq \bE_{u \sim \gamma} \left[ \int_0^1 \left| |\rho^\alpha_{u\#P_n} - \rho^\alpha_{u\#Q_n}| - |\rho^\alpha_{u\#P} - \rho^\alpha_{u\#Q}| \right| \d \alpha\right]
\end{align*}
By the reverse triangle inequality followed by the triangle inequality,
\begin{equation}
\label{eq:rev_tr_then_tr}
\begin{split}
    \left|| \rho_{u\#P_n}^\alpha - \rho_{u\#Q_n}^\alpha | - | \rho_{u\#P}^\alpha - \rho_{u\#Q}^\alpha | \right|
    &\leq |\rho_{u\#P_n}^\alpha - \rho_{u\#P}^\alpha + \rho_{u\#Q}^\alpha - \rho_{u\#Q_n}^\alpha| \\
    &\leq |\rho_{u\#P_n}^\alpha - \rho_{u\#P}^\alpha|+|\rho_{u\#Q_n}^\alpha - \rho_{u\#Q}^\alpha|,
\end{split}
\end{equation}
and the statement of the lemma follows.
\end{proof}
\begin{lemma}
\label{res:lemma2_for_consistency}
Let $\nu$ be a measure on $[0, 1]$ with density $f_\nu$ bounded above by $C_\nu>0$.
With probability at least $1-\delta/4$, for $C'(\delta) = 2 C_\nu \sqrt{\log(8/\delta)/2}$, it holds that
\begin{align*}
    \ekqd_1(P_n, P; \nu, \gamma) \leq \frac{C'(\delta)}{2} n^{-1/2}
\end{align*}
\end{lemma}
\begin{proof}
Recall that
\begin{equation*}
    \ekqd_1(P_n, P; \nu, \gamma) = \bE_{u \sim \gamma} \left[\tau_1(P_n, P; \nu, u) \right], \qquad \tau_1(P_n, P; \nu, u) = \int_0^1 |\rho_{u\#P_n}^\alpha - \rho_{u\#P}^\alpha| \nu(\d \alpha) .
\end{equation*}
Let $F_{u\#P}$ and $F_{u\#P_n}$ be the CDFs of $u\#P$ and $u\#P_n$ respectively. Then,
\begin{align*}
    \int_0^1 |\rho_{u\#P_n}^\alpha - \rho_{u\#P}^\alpha| \nu(\d \alpha) \leq C_\nu \int_0^1 |\rho_{u\#P_n}^\alpha - \rho_{u\#P}^\alpha| \d \alpha
    &= C_\nu \int_{u(\calX)} |F_{u\#P_n}(t) - F_{u\#P}(t) | \d t \\
    &\leq C_\nu \sup_{t \in u(\calX)} |F_{u\#P_n}(t) - F_{u\#P}(t) |,
\end{align*}
where the last equality is the well known fact that integrated difference between quantiles is equal to integrated difference between CDFs (see, for instance,~\citet[Theorem 2.9]{bobkov2019one}). By the Dvoretzky-Kiefer-Wolfowitz inequality, with probability at least $1-\delta/4$ it holds that,
\begin{equation*}
    \sup |F_{u\#P_n}(t) - F_{u\#P}(t)| < \sqrt{ \log(8/\delta) / 2} n^{-1/2},
\end{equation*}
and therefore, with probability at least $1-\delta/4$ for $C'(\delta) = 2 C_\nu \sqrt{ \log(8/\delta) / 2}$,
\begin{equation*}
    \tau_1(P_n, P; \nu, u) = \int_0^1 |\rho_{u\#P_n}^\alpha - \rho_{u\#P}^\alpha| \nu(\d \alpha) \leq \frac{C'(\delta)}{2} n^{-1/2}.
\end{equation*}
In other words, the random variable $\tau_1(P_n, P; \nu, u)$ is sub-Gaussian with sub-Gaussian constant $C_\tau \coloneqq C_\nu^2/(2n)$, meaning
\begin{equation*}
    \mathrm{Pr} \left[\tau_1(P_n, P; \nu, u) \geq \varepsilon\right] \leq 2 \exp\{- \varepsilon^2 / C_\tau^2\}
\end{equation*}
One of the equivalent definitions for a sub-Gaussian random variable is the moment condition: for any $p \geq 1$,
\begin{equation*}
    \bE_{x_{1:n}} \left[\tau_1(P_n, P; \nu, u)^p \right] \leq 2 C_\tau^p \Gamma(p/2 + 1).
\end{equation*}
An application of Jensen inequality and Fubini's theorem shows that the moment condition holds for $\bE_{u \sim \gamma} \tau_1(P_n, P; \nu, u)$,
\begin{equation*}
    \bE_{x_{1:n}} \left[ \left(\bE_{u \sim \gamma} \tau_1(P_n, P; \nu, u)\right)^p \right] \leq \bE_{x_{1:n}} \bE_{u \sim \gamma} \left[\tau_1(P_n, P; \nu, u)^p \right]= \bE_{u \sim \gamma} \bE_{x_{1:n}} \left[\tau_1(P_n, P; \nu, u)^p \right] \leq 2 C_\tau^p \Gamma(p/2 + 1).
\end{equation*}
Therefore, $\bE_{u \sim \gamma} \tau_1(P_n, P; \nu, u)$ is sub-Gaussian with constant $C_\tau = C_\nu^2/(2n)$, meaning it holds with probability at least $1 - \delta/4$ that
\begin{equation*}
    \ekqd_1(P_n, P; \nu, \gamma) = \bE_{u \sim \gamma} \tau_1(P_n, P; \nu, u) \leq \frac{C'(\delta)}{2} n^{-1/2}.
\end{equation*}
\end{proof}

We are now ready to prove the full result.
\begin{proof}[Proof of~\Cref{res:consistency_KQD}]
Let $C_\nu$ be an upper bound on the density of $\nu$.
By triangle inequality, the full error can be upper bounded by $R_l$, the error due to approximation of $\gamma$ with $\gamma_l$, plus $R_n$, the error due to approximation of $P, Q$ with $P_n, Q_n$,
\begin{align*}
    | \ekqd_1(P_n, Q_n;\nu, \gamma_l) - \ekqd_1(P, Q;\nu, \gamma) | &\leq | \ekqd_1(P_n, Q_n;\nu, \gamma_l) - \ekqd_1(P_n, Q_n;\nu, \gamma) | \\
    &\hspace{2cm}+ | \ekqd_1(P_n, Q_n;\nu, \gamma) - \ekqd_1(P, Q;\nu, \gamma) | \\
    &\eqcolon R_l + R_n.
\end{align*}
We bound $R_l$ in $l^{-1/2}$, and $R_n$ in $n^{-1/2}$, with high probability.

\paragraph{Bounding $R_l$.}
Recall that $\ekqd_1(P_n, Q_n; \nu, \gamma) = \bE_{u \sim \gamma} \left[ \int_0^1 | \rho_{u\#P}^\alpha - \rho_{u\#Q}^\alpha | \nu(\d \alpha) \right]$. Therefore, we may apply McDiarmid's inequality provided for any $u, u' \in S_\calH$ we upper bound the difference
\begin{equation*}
    \left| \int_0^1 \left| \rho_{u\#P}^\alpha - \rho_{u\#Q}^\alpha \right| - \left| \rho_{u'\#P}^\alpha - \rho_{u'\#Q}^\alpha \right| \nu(\d \alpha) \right|.
\end{equation*}
We have that
\begin{align*}
    \left| \int_0^1 \left| \rho_{u\#P}^\alpha - \rho_{u\#Q}^\alpha \right| - \left| \rho_{u'\#P}^\alpha - \rho_{u'\#Q}^\alpha \right| \nu(\d \alpha) \right|
    &\stackrel{(A)}{\leq} \int_0^1 \left| \rho_{u\#P}^\alpha - \rho_{u\#Q}^\alpha \right| \nu(\d \alpha) + \int_0^1 \left| \rho_{u'\#P}^\alpha - \rho_{u'\#Q}^\alpha \right| \nu(\d \alpha) \\
    &\stackrel{(B)}{\leq} 2 C_\nu \sup_{u \in S_\calH} W_1(u\#P, u\#Q) \\
    &\stackrel{(C)}{\leq} 2 C_\nu \sup_{u \in S_\calH} \bE_{X \sim P} \bE_{Y \sim Q} |u(X) - u(Y)|
    \\
    &\stackrel{(D)}{\leq} 2 C_\nu \bE_{X \sim P} \bE_{Y \sim Q} \sqrt{k(X, X) - 2 k(X, Y) + k(Y, Y)}
\end{align*}
where $(A)$ holds by Jensen's and triangle inequalities; $(B)$ uses boundedness of the density of $\nu$ by $C_\nu$ and the property of the Wasserstein distance in $\Re$ from~\Cref{eq:wasserstein_as_quantiles}; $(C)$ uses the infimum definition of the Wasserstein distance; and $(D)$ holds by the reasoning we employed multiple times through the paper, via reproducing property, Cauchy-Schwarz, and having $u,u' \in S_\calH$.
So we arrive at a bound
\begin{equation*}
    \left| \int_0^1 \left| \rho_{u\#P}^\alpha - \rho_{u\#Q}^\alpha \right| - \left| \rho_{u'\#P}^\alpha - \rho_{u'\#Q}^\alpha \right| \nu(\d \alpha) \right| \leq 2 C_\nu \bE_{X \sim P} \bE_{Y \sim Q} \sqrt{k(X, X) - 2 k(X, Y) + k(Y, Y)} \eqcolon 2 C_\nu C_k.
\end{equation*}
Now that boundedness of the difference has been established, by McDiarmid's inequality, with probability at least $1 - \delta/2$ and for $C''(\delta) = \sqrt{2 C_\nu C_k \log(4 / \delta)}$ it holds that
\begin{equation*}
    | \ekqd_1(P_n, Q_n;\nu, \gamma_l) - \ekqd_1(P_n, Q_n;\nu, \gamma)| \leq C''(\delta)l^{-1/2}.
\end{equation*}
\paragraph{Bounding $R_n$.} By~\Cref{res:lemma_for_consistency},
\begin{align*}
    |\ekqd_1(P_n, Q_n; \nu, \gamma) - \ekqd_1(P, Q; \nu, \gamma)|
    \leq \ekqd_1(P_n, P; \nu, \gamma) + \ekqd_1(Q_n, Q; \nu, \gamma))
\end{align*}
By~\Cref{res:lemma2_for_consistency} and the union bound, with probability at least $1-\delta/2$ and for $C'(\delta) = 2C_\nu \sqrt{\log(8/\delta)/2}$, it holds that
\begin{align*}
    R_n = |\ekqd_1(P_n, Q_n; \nu, \gamma) - \ekqd_1(P, Q; \nu, \gamma)| \leq C'(\delta) n^{-1/2}.
\end{align*}
\paragraph{Combining bounds.} By applying the union bound again, to $R_l + R_n$, we get that, with probability at least $1-\delta$,
\begin{equation*}
    | \ekqd_1(P_n, Q_n;\nu, \gamma_l) - \ekqd_1(P, Q;\nu, \gamma) | \leq R_l + R_n \leq C''(\delta)l^{-1/2} + C'(\delta) n^{-1/2} \leq C(\delta)(l^{-1/2} + n^{-1/2}),
\end{equation*}
for $C(\delta) = \max\{C'(\delta), C''(\delta)\} = \bigo(\sqrt{\log(1/\delta)})$. This completes the proof
\end{proof}
As pointed out in the main text, $\bE_{X \sim P} \bE_{Y \sim Q} \sqrt{k(X, X) - 2 k(X, Y) + k(Y, Y)}<\infty$ holds immediately when $\bE_{X \sim P}\sqrt{k(X, X)}$ and $\bE_{X \sim Q}\sqrt{k(X, X)}$ are finite, and even more specifically, when the kernel $k$ is bounded. Unbounded $k$ and finite expectations, for example, happens when the tails of both $P$ and $Q$ decay fast enough to "compensate" for the growth of $k(x, x)$. For instance, when $k$ is a polynomial kernel of any order (which is unbounded), and $P$ and $Q$ are laws of sub-exponential random variables. For clarity, note that $\bE_{X \sim P} \bE_{Y \sim Q} \sqrt{k(X, X) - 2 k(X, Y) + k(Y, Y)}$ does not compare to MMD, which integrates $k(X, X')$ rather than $k(X, X)$ (see~\Cref{eq:mmd}).

For integer $p>1$, proving the $n^{-1/2}$ convergence rate is feasible if more involved---primarily because we can no longer reduce the problem to convergence of empirical CDFs to true CDFs. In general, for $p>1$,
\begin{align*}
    \int_0^1 |\rho_{u\#P_n}^\alpha - \rho_{u\#P}^\alpha|^p \d \alpha \neq \int_{u(\calX)} |F_{u\#P_n}(t) - F_{u\#P}(t) |^p \d t.
\end{align*}
The following result, restated in our notation, makes the added complexity explicit.
\begin{lemma}[Theorem 5.3 in~\citet{bobkov2019one}]
\label{res:wasserstein_convergence}
Suppose $k: \Re^d \times \Re^d \to \Re$ is a bounded kernel, and $\nu$ has a density $0<c_\nu \leq f_\nu \leq C_\nu$ on $[0,1]$. Then, for any $u \in S_\calH$, and for any $p \geq 1$ and $n \geq 1$,
\begin{equation*}
    \bE_{x_{1:n} \sim P} \left[ \tau_p^p(P_n, P; \nu, u) \right] \leq \left(\frac{5p C_\nu}{\sqrt{n+2}}\right)^p J_p(u\#P), \quad \text{ for } \quad J_p(u\#P) = \int_{u(\calX)} \frac{\left( F_{u\#P}(t) (1 - F_{u\#P}(t)) \right)^{p/2}}{f^{p-1}_{u\#P}(x)}\d t.
\end{equation*}
Further, it holds that $\bE_{x_{1:n} \sim P} \left[\tau_p^p(P_n, P; \nu, u)=\bigo(n^{-p/2}) \right]$ if and only if $J_p(u\#P) < \infty$.
\end{lemma}

We now state a likely result for $p>1$ as a conjecture, and outline the proof.
\begin{conjecture}[\textbf{Finite-Sample Consistency for Empirical KQDs for $p>1$}]
\label{res:consistency_KQD_p}
    Let $\calX \subseteq \Re^d$, $\nu$ have a density, $P, Q$ be measures on $\calX$ with densities bounded away from zero, $f_P(x) \geq c_P>0$ and $f_P(x) \geq c_Q>0$. Suppose $\bE_{X \sim P} [k(X, X)^{\nicefrac{p}{2}}]<\infty$ and $\bE_{X \sim Q} [ k(X, X)^{\nicefrac{p}{2}}]<\infty$, and $x_{1:n} \sim P, y_{1:n} \sim Q$. Then,
    \begin{align*}
        \bE_{\substack{x_{1:n}\sim P \\ y_{1:n}\sim Q}}| \ekqd_p(P_n, Q_n;\nu, \gamma_l) - \ekqd_p(P, Q;\nu, \gamma) | =\bigo(l^{-\nicefrac{1}{2}} + n^{-\nicefrac{1}{2}}).
    \end{align*}
\end{conjecture}
\begin{proof}[Sketch proof]
    Analogously to the proof of~\Cref{res:consistency_KQD}, we can decompose the term of interest as
    \begin{align*}
        &\bE_{\substack{x_{1:n}\sim P \\ y_{1:n}\sim Q}}| \ekqd_p(P_n, Q_n;\nu, \gamma_l) - \ekqd_p(P, Q;\nu, \gamma) | \\
        &\hspace{2cm}\leq 
        \bE_{\substack{x_{1:n}\sim P \\ y_{1:n}\sim Q}}| \ekqd_p(P_n, Q_n;\nu, \gamma_l) - \ekqd_p(P_n, Q_n;\nu, \gamma) |\\
        & \hspace{3cm} +\left(\bE_{x_{1:n}\sim P} \ekqd^p_p(P_n, P; \nu, \gamma) \right)^{\nicefrac{1}{p}} + \left(\bE_{y_{1:n}\sim Q}\ekqd^p_p(Q_n, Q; \nu, \gamma))\right)^{\nicefrac{1}{p}}
    \end{align*}
    The first term can be, same as in the proof of~\Cref{res:consistency_KQD}, bounded by McDiarmid's inequality. The second term (to the power $p$) takes the form 
    \begin{equation*}
        \bE_{x_{1:n}\sim P} \ekqd^p_p(P_n, P; \nu, \gamma) = \bE_{x_{1:n}\sim P} \bE_{u\sim \gamma} \tau_p^p(P_n, P; \nu, u).
    \end{equation*}
    Then, by~\Cref{res:wasserstein_convergence} (possibly modified to account for an extra expectation), to get the result we will need to show that $\bE_{u \sim \gamma} J_p(u\#P) < \infty$,
    \begin{equation*}
        \bE_{u \sim \gamma} J_p(u\#P) = \bE_{u \sim \gamma} \left[\int_{u(\calX)} \frac{\left( F_{u\#P}(t) (1 - F_{u\#P}(t)) \right)^{p/2}}{f^{p-1}_{u\#P}(x)}\d t \right] < \infty
    \end{equation*}
    The nominator is upper bounded by $2^{-p}$. The denominator may get arbitrarily small without the nominator getting arbitrarily small: when the PDF $f^{p-1}_{u\#P}(x)$ is small, the CDF $F_{u\#P}(x)$ need not be close to zero or one. Therefore, it is necessary and sufficient to show
    \begin{equation}
    \label{eq:one_over_density}
        \bE_{u \sim \gamma} \left[\int_{u(\calX)} \frac{1}{f^{p-1}_{u\#P}(x)}\d t\right] < \infty.
    \end{equation}
    We proceed to outline key elements of the proof of such result, and leave a rigorous proof for future work. By the coarea formula, and since $f_P(x) \geq c_P > 0$,
    \begin{equation*}
        f_{u\#P}(t) = \int_{u^{-1}(t)} \frac{f_P(x)}{|\nabla u(x)|} H^{d-1}(\d x) \geq c_0 \int_{u^{-1}(t)} \frac{1}{|\nabla u(x)|} H^{d-1}(\d x), \qquad \text{for} \qquad |\nabla u(x)| = \sqrt{\sum_{i=1}^d \left(\frac{\partial u(x)}{\partial x_i} \right)^2}
    \end{equation*}
    where $u^{-1}(t)=\{x \in \calX : u(x)=t\}$, and $H^{d-1}$ is the $d-1$-dimensional Hausdorff measure, which within $\calX \subseteq \Re^d$ is equal to $d-1$ dimensional Lebesgue measure, scaled by a constant that only depends on $d-1$.
    
    Therefore, the integral in~\Cref{eq:one_over_density} may diverge if the integral
    \begin{equation}
    \label{eq:volume_of_level_set}
        \int_{u^{-1}(t)} \frac{1}{|\nabla u(x)|} H^{d-1}(\d x)
    \end{equation}
    gets very small over "large" parts of $u(\calX)$---on average over $u \sim \gamma$. Trivially, if $u$ is constant over some interval---or more generally, $u$ has infinitely many critical points---the integral diverges. Fortunately, the more general condition is easy to control: if $u$ is a \emph{Morse function} and $\calX$ is compact, then $u$ has only a finite number of critical points. It is a classic result (see, for instance,~\citet[Theorem 1.2]{Hirsch1976}) that Morse functions form a dense open subset of twice differentiable real-valued functions on $R^d$, denoted $C^2(\Re^d)$. Therefore, if $\calH \subset C^2(\calX)$ (which can be reduced to smoothness of the kernel $k$---it holds for instance, for the Mat\'ern-5/2 kernel), we get that $u \sim \gamma$ has a finite number of critical points almost surely under mild regularity assumptions on $\gamma$.
    
    The final ingredient is to use the Morse lemma to lower bound~\Cref{eq:volume_of_level_set} in the epsilon-ball of each critical point. Morse lemma says $u$ is quadratic around each critical point---which yields bounds on both the volume of $u^{-1}(t)$, and $1/|\nabla u(x) |$ in terms of the eigenvalues of the Hessian. Careful analysis of the eigenvalues will be needed to ensure the expectation with respect to $u \sim \gamma$ is finite.
\end{proof}

\subsection{Proof of \Cref{res:connections_slicedwasserstein,res:connections_maxslicedwasserstein}}
\label{appendix:proof_connections_slicedwasserstein}

The equality in~\Cref{eq:wasserstein_as_quantiles} immediately gives the connection of $\ekqd$ and $\supkqd$ to the expected-SW and max-SW respectively---previously only defined on $\calX=\Re^d$.

Further, for $\calX=\Re^d$, viewing $x \mapsto k(x, \cdot)$ as a transformation on $\calX$ reveals a connection to Generalised Sliced Wasserstein (GSW, \citet{kolouri_generalized_2022}). In particular, the polynomial kernel $k(x, x') = (x^\top x' + 1)^T$ of odd degree $T$ recovers the polynomial transformation for which GSW was proven to be a probability metric. Outside of the case of the polynomial case, proving that GSW is a metric is highly challenging. This is easier under the kernel framework, as we showed in~\Cref{res:KQE_characterise_dists}.
In~\citet{kolouri_generalized_2022}, the authors investigate learning transformations with neural networks (NNs). An interesting direction for future work is the relationship between said NNs and the kernels they induce.

\subsection{Proof of~\Cref{res:sampling_from_gm}}
\label{sec:proof_sampling_from_gm}
Recall that by definition of Gaussian measures in Hilbert spaces~\citep{kukush2020gaussian}, a random element $f \in \calH$ has the law of a Gaussian measure $\calN(0, C_m)$ on $\calH$ when for any $g \in \calH$,
\begin{equation}
\label{eq:defn_of_gm}
    \langle f, g \rangle_\calH \sim \calN(0, \langle C_m[g], g\rangle).
\end{equation}
Since $C_m[g](x) = \nicefrac{1}{m} \sum_{j=1}^m g(z_j) k(z_j, x)$, by the reproducing property,
\begin{equation}
\label{eq:empirical_integral_op_explicit}
    \langle C_m[g], g\rangle = \frac{1}{m}\sum_{j=1}^m g(z_i)^2.
\end{equation}
Take $f(x) = \nicefrac{1}{\sqrt m} \sum_{j=1}^m \lambda_j k(z_j, x)$, for $\lambda_1, \dots, \lambda_m \sim \calN(0, \Id)$. Then, for any $g \in \calH$, by the reproducing property it holds that
\begin{equation*}
    \langle f, g \rangle_\calH = \frac{1}{\sqrt m} \sum_{j=1}^m \lambda_j g(z_j) \sim \calN\left(0, \frac{1}{m} \sum_{i=1}^m g(z_i)^2\right),
\end{equation*}
which is exactly the Gaussian measure with covariance operator $C_m$, as per~\Cref{eq:defn_of_gm,eq:empirical_integral_op_explicit}.

\section{Additional Numerical Results}\label{appendix:experiments}

\subsection{Type I control}

We report the Type I control experiments for the CIFAR-10 v.s. CIFAR-10.1 experiment. Results are shown in Figure \ref{fig:type-1-result}.

\begin{figure}
    \centering
    \includegraphics[width=0.5\linewidth]{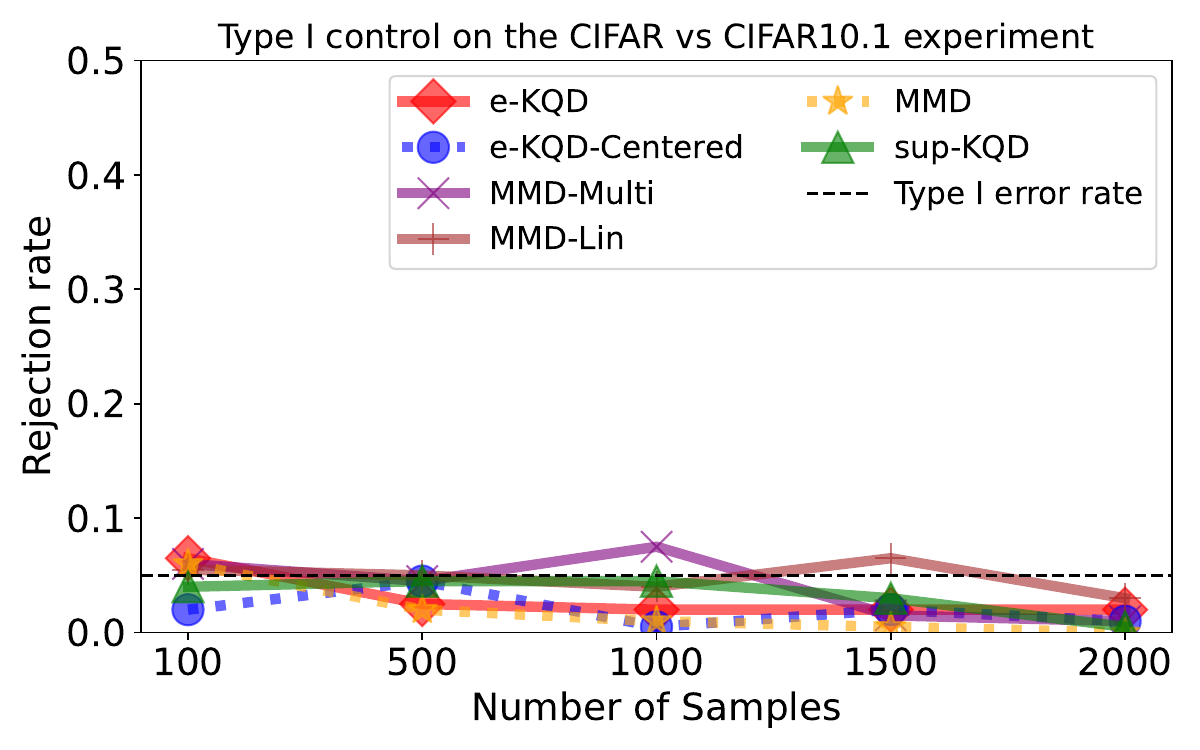}
    \caption{Type I control results for our experiment on CIFAR-10 v.s. CIFAR-10.1. We see all methods control their Type I error around or below the specified Type I error rate $0.05$, thus confirming our tests in the main text are valid testing procedures.}
    \label{fig:type-1-result}
\end{figure}

\subsection{\Cref{fig: main_experiments} for $\ekqd_1$}
\label{sec:exp_for_p1}

It is common in power $p$-parametrised methods to select $p=2$, to balance out sensitivity to outliers (which is higher for larger $p$, to the point of methods becoming brittle for $p > 2$), and robustness (which tends to be highest for $p=1$); this trade-off, for instance, inspired the introduction of the Huber loss~\citep{huber1964robust}. However, for completeness, we now repeat experiments in the main paper for $p=1$. The relationship to baseline approaches---MMD, MMD-Multi, and MMD-Lin---remains the same as observed for $p=2$. However, it is evident that $\ekqd_1$ performed better than $\ekqd_2$ at the power decay and galaxy MNIST experiments, but the centered $\ekqd_1$ performed worse than centered $\ekqd_2$ at the Laplace v.s. Gaussian experiment. The implications of choosing $p$ warrants a deeper investigation, left to future work.

\begin{figure*}[t]
    \centering
    \includegraphics[width=\linewidth]{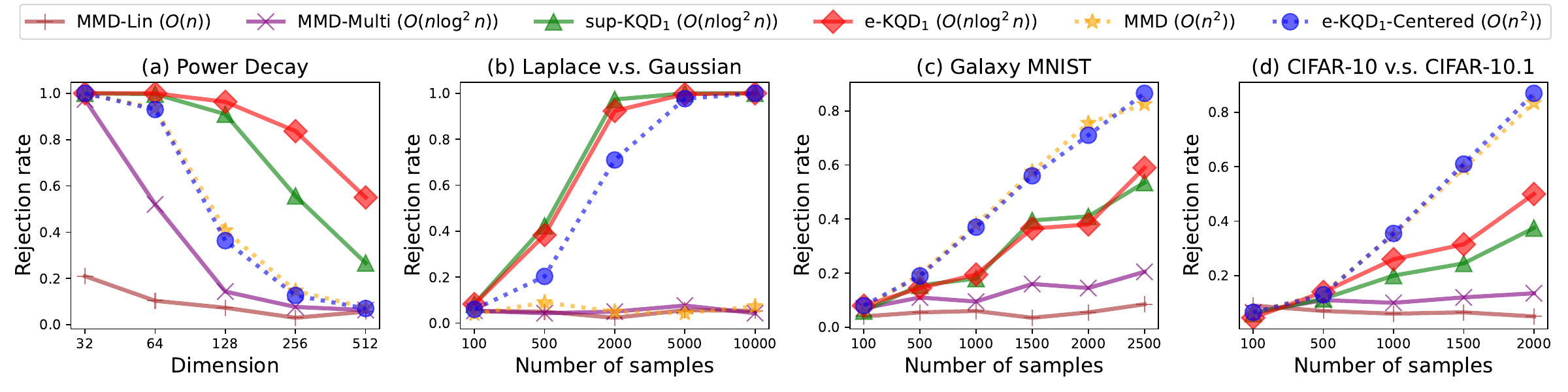}
    \caption{The experiments in~\Cref{fig: main_experiments} repeated for $p=1$. Experimental results comparing our proposed methods with baseline approaches. A higher rejection rate indicates better performance in distinguishing between distributions. \textbf{Same as for $p=2$, quadratic-time quantile-based estimators perform comparably to quadratic-time MMD estimators, while near-linear time quantile-based estimators often outperform their MMD-based counterparts.}}
    \label{fig: main_experiments_for_p1}
\end{figure*}

\subsection{Comparison of weighting measures}

The Gaussian Kernel Quantile Discrepancy introduced in~\Cref{sec:estimator} has multiple weighting measures that determine properties of the distance: the measure $\nu$ on the quantile levels, the measure $\xi$ within the covariance operator, and the measure $\gamma$ on the unit sphere $S_\calH$. We investigate the impact of varying these.

\paragraph{Varying $\nu$.} We conducted the following experiment using the Galaxy MNIST and CIFAR datasets. We varied $\nu$, from assigning more weight to the extreme quantiles to down-weighting them. The results are presented in~\Cref{fig:varying_combined}, where the reverse triangle $\setminus/$ stands for up-weighing extreme quantiles, and the triangle $/\setminus$ stands for down-weighing them. We observed some improvement over the uniform $\nu$: for Galaxy MNIST, test power improved when $\nu$ assigned less weight to extremes, whereas for CIFAR, the opposite was true, with higher test power when more weight was given to extremes. Uniform weighting of the quantiles remained a good choice. This suggests that tuning $\nu$ beyond the uniform is problem-dependent and can enhance performance. The difference likely arises from the nature of the problems: CIFAR datasets, where samples are expected to be similar, benefit from emphasising extremes, while Galaxy MNIST, which contains fundamentally different galaxy images, performs better when “robustified,” i.e., focusing on differences away from the tails. Exploring this further presents an exciting avenue for future work.

\begin{figure*}[t]
    \centering
    \includegraphics[width=0.8\linewidth]{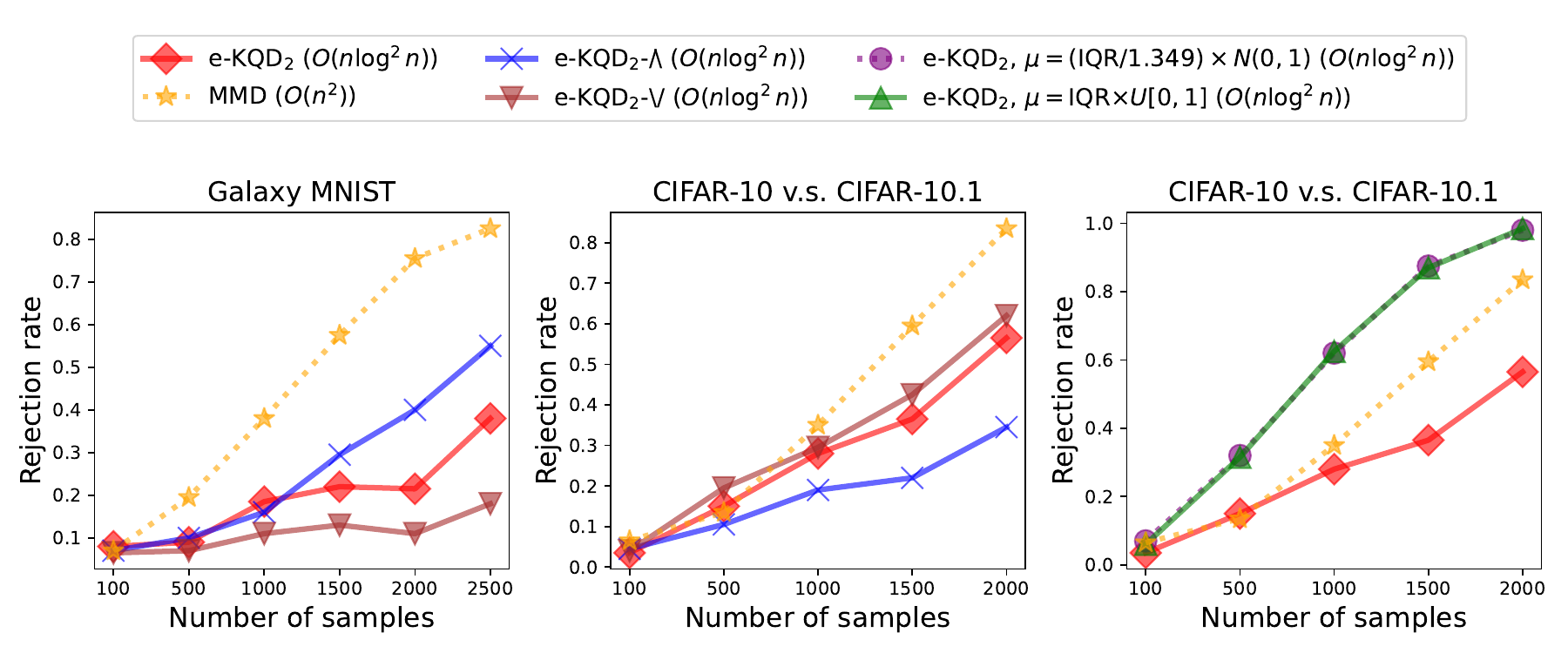}
    \caption{Gaussian KQD test power under different weighting measures. \textit{Left, middle:} Varying measure $\nu$: down-weighing ($/\setminus$) extremes boosts power on Galaxy MNIST, while up-weighing ($\setminus/$) them helps on CIFAR. Uniform weighting remains a strong default, with optimal $\nu$ depending on the dataset.
    \textit{Right:} Varying measure $\xi$: using an IQR-scaled Gaussian or uniform default reference measure $\xi$ both outperform MMD---indicating potential advantage of a "default" $\xi$ over the problem-based $\xi=(P_n + Q_n)/2$.}
    \label{fig:varying_combined}
\end{figure*}

\paragraph{Varying $\xi$.} The reference measure $\xi$ in the covariance operator $C$ serves to "cover the input space" and is typically set to a "default" measure on the space---for $\Re^d$, the standard Gaussian measure. The choice $(P_n + Q_n)/2$ made in the main body of the paper is aiming to adhere to the most general setting, when no default measure may be available---only $P_n$ and $Q_n$.

We report a comparison on performance when the reference measure is: (1) $(P_n + Q_n)/2$; (2) a standard Gaussian measure, scaled by IQR/1.349 to match the spread of the data, where IQR is the interquantile range of $P_n + Q_n$, and 1.349 is the interquantile range of the standard Gaussian; and
(3) a uniform measure on $[-1, 1]^d$, scaled by IQR.

The results, presented in~\Cref{fig:varying_combined}, show performance superior to MMD for the standard/uniform $\xi$. This indicates value in picking a "default" measure when one is available.

\paragraph{Varying $\gamma$} Varying the measure on the sphere beyond a Gaussian is extremely challenging in infinite-dimensional spaces due to the complexity of both its theoretical definition and practical sampling. Since no practically relevant alternative has been proposed, we leave this direction unexplored.

\subsection{Comparison to sliced Wasserstein distances}

We extend the power decay experiment to include sliced Wasserstein and max-sliced Wasserstein distances, with directions (1) sampled uniformly on the sphere, and (2) sampled from $(P_n + Q_n)/2$ and projected onto the sphere. The results are plotted in~\Cref{fig:sw_and_kme_approx}, and show that sliced Wasserstein distances perform significantly worse than $\ekqd$. This outcome is expected---as noted in~\Cref{res:connections_slicedwasserstein,res:connections_maxslicedwasserstein}, sliced Wasserstein is equivalent to $\ekqd$ with the linear kernel, which is less expressive than the Gaussian kernel.

\begin{figure*}[t]
    \centering
    \includegraphics[width=0.8\linewidth]{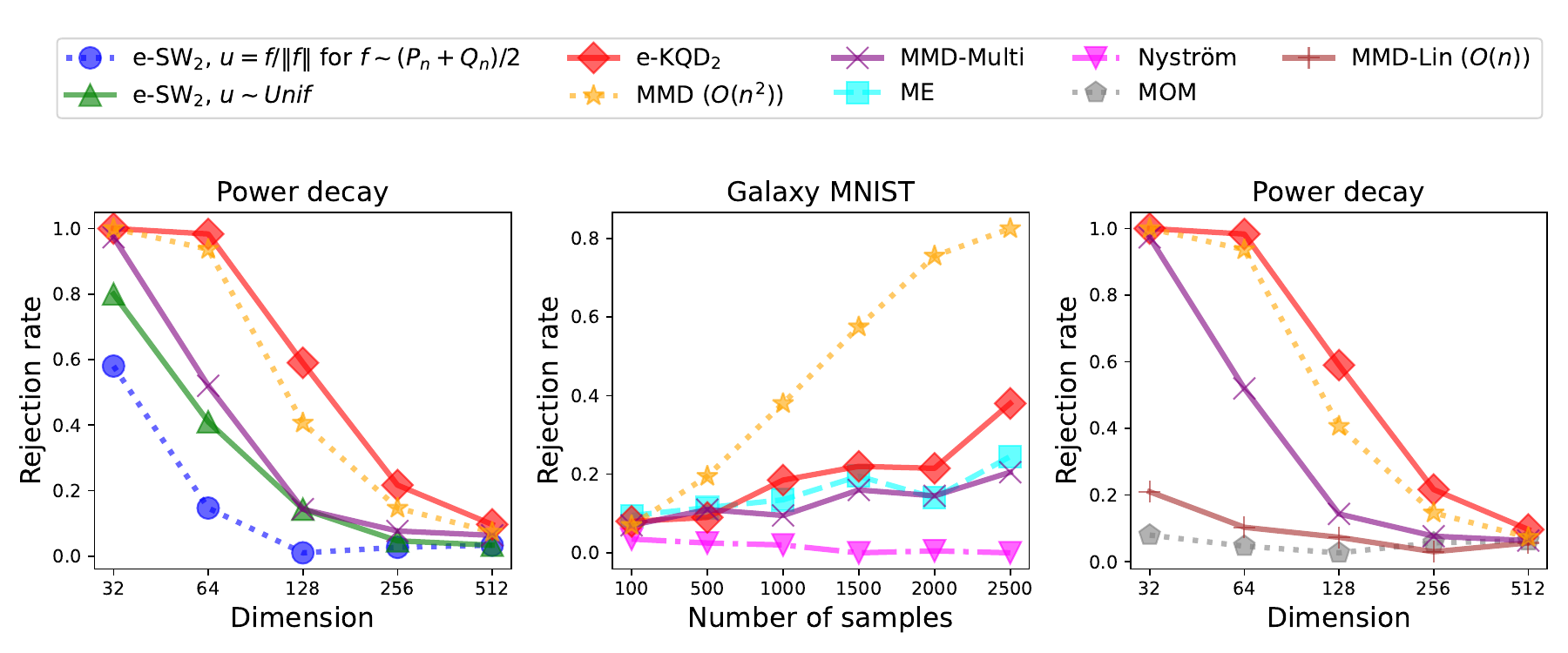}
    \caption{All methods are cost $\bigo(n \log^2 n)$ unless specified otherwise. \textit{Left:} Gaussian KQD compared with sliced Wasserstein with uniform or data-driven directions, on the power decay problem. Sliced Wasserstein fall well below KQD---consistent with their equivalence to KQD using a less expressive linear kernel.
    \textit{Middle:} Comparison with alternative approximate KME methods, at matching cost. ME matches MMD-multi power, while Nyström-MMD suffers high Type II error.
    \textit{Right:} Comparison with Median-of-Means (MOM) KME approximation, at matching cost. MOM is primarily a robustness-enforcing method, not a cheap-approximation method, and doesn't perform well at set cost of $\bigo(n \log^2 n)$.}
    \label{fig:sw_and_kme_approx}
\end{figure*}

\subsection{Comparison with MMD based on Other KME Approximations}

There are several efficient kernel mean embedding methods available in the literature, and no single approach has emerged as definitively superior. To complement experiments in the main body of the paper, we compare the $\ekqd$ (at matching cost) with (1) The Mean Embedding (ME) approximation of MMD of~\citet{chwialkowski2015fast}, which was identified as the best-performing method in their numerical study; (2) the Nystr\"om-MMD method of~\citet{chatalic2022nystrom}, and (3) the Median-of-Means (MOM) approximation of~\citet{lerasle2019monk}, specifically, their faster method (MONK BCD-Fast) that achieves matching cost to our e-KQD at the number of blocks $Q=n/\log n$.

The results are presented in~\Cref{fig:sw_and_kme_approx}. ME performs at the level of MMD-multi, while Nystr\"om has extremely high Type II error, likely due to sensitivity to hyperparameters. Due to Median-of-Means still being considerably slower than e-KQD (with the number of optimiser iterations set to $T=100$), we apply it to a cheaper Power Decay problem (rather than the larger and more complicated Galaxy MNIST), where it performs at the level of the linear approximation of MMD. This may be due to MOM primarily being robustness-enforcing method, rather than a method aiming to build an efficient approximation of MMD.

\end{document}